\documentclass{article}

\usepackage{arxiv}

\usepackage[utf8]{inputenc} 
\usepackage[T1]{fontenc}    
\usepackage{url}            
\usepackage{nicefrac}       

\usepackage{graphicx}
\usepackage{natbib} 
\usepackage{multibib}
\usepackage{hyperref} 
\hypersetup{ 
    colorlinks=true,
    linkcolor=blue,  
    citecolor=blue,  
    urlcolor=blue,  
    }

\usepackage{amsmath}
\usepackage{amsfonts} 
\usepackage{bm} 
\usepackage{dsfont} 

\usepackage{amsthm} 
\newtheorem{remark}{Remark}[section]%

\newtheorem{assumption}{Assumption}

\newtheorem{theorem}{Theorem}[section]
\newtheorem{lemma}{Lemma}[section]

\usepackage{booktabs}  
\usepackage{adjustbox}
\usepackage{multirow}
\usepackage{multicol}
\usepackage{makecell}
\usepackage{rotating} 

\usepackage{caption}  
\usepackage{subcaption}   
\usepackage{mwe}
\newlength{\tempheight}
\newlength{\tempwidth}
\newcommand{\rowname}[1]
{\rotatebox{90}{\makebox[\tempheight][c]{{#1}}}}
\newcommand{\columnname}[1]
{\makebox[\tempwidth][c]{{#1}}}

\usepackage{algorithm} 
\usepackage{algorithmicx} 
\usepackage{algpseudocode} 

\title{MNAR-$k$-means: A $k$-means Clustering for Data Missing Not at Random with Magnitude-Decaying Probability}

\author{
 Xin Guan \\
  Graduate School of Information Sciences\\
  Tohoku University\\
  Sendai, Miyagi 980-8579, Japan\\
  \texttt{guan.xin.c5@tohoku.ac.jp} \\
}

\begin{document}
\maketitle

\begin{abstract}
The classical $k$-means clustering, based on distances computed from all data features, cannot be directly applied to incomplete data with missing values. A natural extension of $k$-means to missing data is to involve only the observed positions in clustering, which is equivalent to imputing missing values by corresponding cluster means. 
However, for data missing not at random (MNAR), since missingness is related to data values, such a mean-imputation-based method may lead to the distortion of estimated cluster centers, resulting in a poor clustering result. 
Since MNAR mechanisms are very common in reality, it is necessary to improve the performance of $k$-means-based clustering methods for such data.  
In this paper, we focus on a magnitude-decaying MNAR scenario where data is more likely to be missing at positions with smaller absolute values, and we propose a novel $k$-means clustering method based on the constraint of the size of imputation values, which enjoys a good mathematical interpretation. 
Moreover, we establish the statistical consistency of the estimated cluster centers of the proposed method to the true cluster centers of fully observed data, and solve the optimization of the proposed loss function via an alternative minimization algorithm. 
Simulation experiments verify the effect of the proposed method in improving clustering results and reducing the bias of estimated cluster centers. 
Applications to real-world missing data further show the utility of the proposed method.   
\end{abstract}

\section{Introduction}

The $k$-means clustering is one of the most widely used clustering methods, which gives a partition for a dataset based on the nearest cluster center of each data point. However, the requirement for a fully observed dataset of $k$-means limits its capacity for missing data. 

In reality, the missingness of data is usually caused by many reasons, which is called the \textit{missing mechanism}. According to \cite{little2019statistical}, there are three main types of missing mechanisms: missing completely at random (MCAR), missing at random (MAR), and missing not at random (MNAR). 
The MCAR mechanism indicates that the missingness of data is independent to the data value. The MAR mechanism indicates that the missingness of each position depends on the observed part of data. The MNAR mechanism indicates that the missingness of each position depends on the corresponding data value. Usually, the MCAR and MAR mechanisms are collectively referred to as ignorable missingness, while the MNAR mechanism is often non-ignorable. 

To apply $k$-means to missing data, the traditional strategy is two-stage, that is, apply missing data handling techniques to construct a new complete data matrix first, and then conduct the $k$-means clustering. 
The simplest method is complete-case analysis that excludes incomplete data points from analysis, which does not work when the missing proportion is large \cite{hathaway2001fuzzy}. 
The classical imputation methods fill in missing entries by a single value estimated from observed entries, such as mean imputation based on the sample mean in each feature, and regression imputation based on a regression model on observed features, whereas these methods may lead to a bias of imputation values if the hidden probabilistic model about missingness is complicated \cite{ghahramani1995learning}. 
Some more sophisticated methods can achieve more accurate imputation. For example, multiple imputation gives a set of imputation values and then pool the imputation results to reduce the uncertainty of single imputation \cite{buuren2011mice,honaker2011amelia}; imputation using machine learning and deep learning algorithms can also obtain a high accuracy, including $K$-nearest neighbors \cite{Hulse2014,liu2016knn}, random forest \cite{stekhoven2012missforest}, traditional neural network \cite{Choudhury2019nn,spinelli2020} and generative adversarial network \cite{yoon2018gain}. 
However, these imputation methods usually pursue a perfect recovery of the original fully observed dataset and thus have a high computation cost.

On the other hand, many integrated methods and algorithms for missing data clustering have also been developed, which mainly provide new measurements for dissimilarity between two incomplete data points, as a substitute of the Euclidean distance used in full observed data clustering. 
For example, \cite{abdallah2014mean} consider the data distribution density in each feature and derive an approximate value of the distance between two incomplete data points that only relies on the means and variances of the data distribution in each feature. 
\cite{datta2018clustering} construct a dissimilarity of two incomplete data points by adding a penalty term to the Euclidean distance in observed features of both points, where the penalty term is given by the sum of observed probabilities in missing features of these two points. 
However, these substitutes of the Euclidean distance are not guaranteed to be a distance measure, limiting further theoretical analysis. 
\cite{mesquita2017} propose a measurement called expected Euclidean distance that is indeed a distance measure, whereas it requires a special data distribution. 
In addition, some other methods consider representing each missing value by an interval instead of a single value, such as \cite{li2013hybrid,li2016robust}, to obtain a more robust clustering result, which often lacks mathematical interpretations and statistical guarantees. 
Furthermore, a few discussions on the missing mechanism are mentioned in these methods.

Apart from the above methods, another simple and intuitive method is to rewrite the $k$-means clustering to be a special matrix decomposition problem, then by introducing a mask matrix to indicate missing positions, we can directly minimize the $k$-means loss of matrix form over observed positions only. 
Such a method is known as $k$-POD (i.e., $k$-means for partially observed data), and was proposed by \cite{chi2016k} and \cite{wang2019k} independently. The $k$-POD clustering is also regarded as a natural and efficient extension of $k$-means to missing data for its easy implementation and flexibility for various missingness, and thus has received much attention. 
For example, \cite{lithio2018efficient} propose a variant of $k$-POD by using the Hartigan-Wong algorithm \cite{hartigan1979} to speed up the original Lloyd's algorithm \cite{Lloyd1982}, and \cite{aschenbruck2023imputation} extend the application of $k$-POD to mixed type data with missingness, and \cite{agliz2025joint} adopt the idea of $k$-POD for dimensional reduction.

However, the $k$-POD clustering performs poorly for missing data under MNAR mechanisms. 
In fact, the $k$-POD clustering is essentially equivalent to filling in missing values of each data point by corresponding values of the cluster center to which each data point belongs. This implies that the clustering performance depends heavily on the accuracy of the estimated cluster centers.  
Moreover, we note that the estimated cluster centers are given by the averages of data points belonging to each cluster, where only observed values are involved in the calculation, and the accuracy of estimated cluster centers may vary under different missing mechanisms. 
For the MCAR mechanism, since missingness is independent with data values, then the dissimilarity between the estimated cluster centers given by imputed data and the true cluster centers given by original fully observed data will converge to zero when the sample size is large enough, as long as cluster labels are estimated correctly. 
However, for the MNAR mechanism, since missingness is non-ignorable, such a property may no longer hold, leading to the bias of estimated cluster centers, and consequently a poor clustering result.

Figure~\ref{fig_intro_example} gives a toy example of data in $\mathbb{R}^2$ with 30\% missingness, showing the failure of $k$-POD clustering method under an MNAR mechanism. 
The left panel is the result of the classical $k$-means clustering on the original fully observed data, which has 39 points with incorrect labels. 
However, under the MNAR mechanism, the $k$-POD method (central panel) gives biased cluster centers and a poor clustering result with 87 data points being mis-clustered.

\begin{figure}[!t]
    \centering
    \begin{minipage}{0.25\textwidth}
    \centering
        \includegraphics[width=\linewidth]{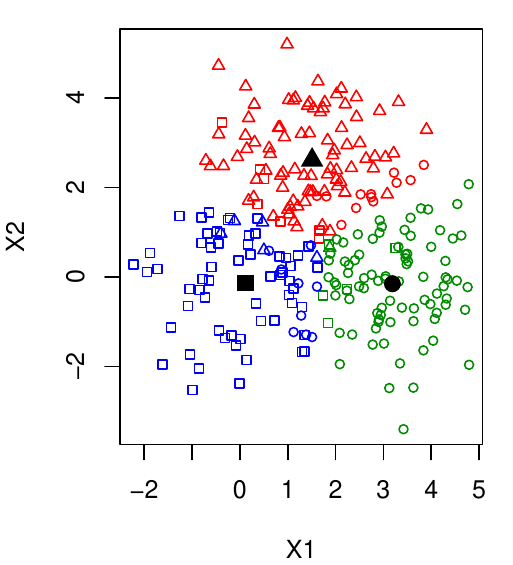}
        {\small $k$-means}
    \end{minipage}
    \begin{minipage}{0.25\textwidth}
    \centering
        \includegraphics[width=\linewidth]{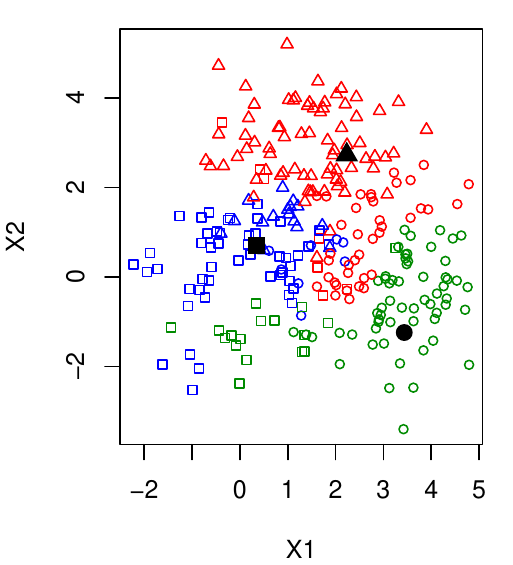}
        {\small $k$-POD}
    \end{minipage}
    \begin{minipage}{0.25\textwidth}
    \centering
        \includegraphics[width=\linewidth]{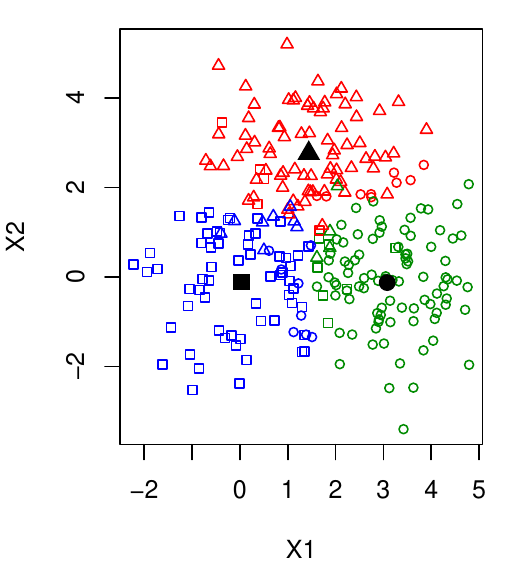}
        {\small Proposed}
    \end{minipage}
    \caption{The clustering results of data in $\mathbb{R}^2$ with 30\% missingness by different methods. From left to right: $k$-means on original fully observed data, $k$-POD and proposed method on missing data under an MNAR mechanism. Each point represents an original fully observed data point. Shape represents the true cluster label, and color represents the estimated cluster label. Three solid black points are estimated cluster centers. }
    \label{fig_intro_example}
\end{figure}

Our aim is to reduce the bias and improve the performance of $k$-means-based clustering methods for data missing not at random. 
Since the MNAR mechanisms are common but usually complex in reality, it is difficult to provide a universal solution applicable for all scenarios. In this work, we focus on a common scenario of data with magnitude-decaying missing probability, where data is more likely to be missing at positions with smaller absolute values. 

For this purpose, we note that in the above scenario, it is more possible that the true value of a missing position is close to zero, and cluster-mean imputation can lead to an overestimation for cluster centers. 
Therefore, it is necessary to control the size of imputation values, which inspires us to penalize the imputed data matrix at missing positions when doing clustering. 
Figure~\ref{fig_intro_example} (right panel) shows a better clustering result for the toy example via penalization, where only 41 data points are mis-clustered (similar to $k$-means). Our contributions are as follows: 
\begin{itemize}
    \item We analyze the reason for the failure of existing $k$-POD clustering under MNAR mechanisms;   
    \item We propose a novel $k$-means clustering method for missing data under a magnitude-decaying MNAR mechanism, which enjoys a good mathematical interpretation; 
    \item We establish the statistical consistency of the estimated cluster centers obtained by the proposed method, and give sufficient conditions under which the estimator converges to the true cluster centers of fully observed data;
    \item We propose an algorithm for optimizing the proposed loss based on alternative minimization, which has a convergence guarantee and linear computation time. 
\end{itemize}

The structure of the rest of this paper is as follows. In Section~\ref{sec_methodology}, we analyze why the $k$-POD clustering fails under MNAR mechanisms, and introduce the proposed method and its mathematical interpretation. In Section~\ref{sec_theoretical_analysis}, we analyze the theoretical properties of the proposed method. In Section~\ref{sec_optimization}, we introduce the optimization of the proposed loss function, including the proposed algorithm and other details of implementation. In Section~\ref{sec_simulations}, we evaluate the effect of the proposed method via numerical experiments on synthetic missing data, and we verify the utility of the proposed method by applying it to real-world missing data in Section~\ref{sec_applications}. Section~\ref{sec_conclusion} is the conclusion and discussions.

\section{Methodology}
\label{sec_methodology}

\subsection{Notations and preliminaries}

Denote by $\bm{X}\in\mathbb{R}^{n\times p}$ the data matrix consisting of $n$ data points $\bm{x}_1,\dots,\bm{x}_n$ in $\mathbb{R}^p$. 
For each $\bm{x}_i=(x_{i1},\dots,x_{ip})$, $i=1,\dots,n$, let $\bm{r}_i=(r_{i1},\dots,r_{ip})\in\{0,1\}^p$ be the indicator of missingness. That is, $r_{ij}=1$ means $x_{ij}$ being observed, and $r_{ij}=0$ missing. 
Write $\bm{U}=(u_{il})_{n\times k} \in \{0,1\}^{n\times k}$, where $u_{il}=1$ if $\bm{x}_i$ belongs to cluster $l$. 
Write $\bm{M}=(\mu_{lj})_{k\times p} \in \mathbb{R}^{k\times p}$, where its $l$-th row $\bm{\mu}_l\in \mathbb{R}^p$ is the $l$-th cluster center. 
Write $\|\bm{A}\|_F=\left(\sum_{i=1}^{n}\sum_{j=1}^{p} a_{ij}^2 \right)^{1/2}$ for Frobenius norm of matrix $\bm{A}\in\mathbb{R}^{n\times p}$. Write $\bm{1}_k$ for all-one vector with length $k$. 
The matrix form of $k$-means clustering is given by
\begin{align}
    \min_{\bm{U},\bm{M}} \|\bm{X} -\bm{UM}\|_F^2,
\label{eq_def_km}
\end{align}
where $\bm{U}\bm{1}_k=\bm{1}_n$. 
Moreover, when the data matrix has missingness, we let $\Omega \subset \{1,\dots,n\}\times \{1,\dots,p\}$, where $(i,j)\in\Omega$ if $x_{ij}$ is observed. 
Define a mapping $\mathcal{P}_{\Omega}:\mathbb{R}^{n\times p}\rightarrow \mathbb{R}^{n\times p}$ by $\big[\mathcal{P}_{\Omega}(\bm{Y})\big]_{ij}=y_{ij}\mathds{1}\left((i,j)\in \Omega\right)$ for any matrix $\bm{Y}=(y_{ij})\in\mathbb{R}^{n\times p}$. 
The $k$-POD clustering is given by 
\begin{align}
    \min_{\bm{U},\bm{M}} \| \mathcal{P}_{\Omega}(\bm{X} -\bm{UM})\|_F^2,
\label{eq_def_kpod}
\end{align}
where $\bm{U}\bm{1}_k=\bm{1}_n$. 
It can be seen that Eq.(\ref{eq_def_kpod}) is equivalent to minimizing $\|\widehat{\bm{X}} -\bm{UM}\|_F^2$ with $\widehat{\bm{X}}=\mathcal{P}_{\Omega}(\bm{X})+\mathcal{P}_{\Omega^c}(\bm{UM})$, where $\Omega^c$ is the complement set of $\Omega$.

\subsection{Why $k$-POD fails for MNAR data}
\label{sec_why_kpod_fails}

To explain the failure of $k$-POD for MNAR data, we consider the full observed data points $\bm{x}_1,\dots,\bm{x}_n$ being independently drawn from the same distribution with density function as follows: 
\begin{align*}
    f(\bm{x}_i ; \bm{U}, \bm{M})=\prod_{l=1}^{k} \phi_{p}(\bm{x}_i;\bm{\mu}_l)^{u_{il}},
\end{align*}
where $\bm{U}=(u_{il})_{n\times k}\in\{0,1\}^{n\times k}$ with $\bm{U}\bm{1}_k=\bm{1}_n$ indicates the cluster membership of each $\bm{x}_i$, and $\phi_{p}(\bm{x}_i;\bm{\mu}_l)$ is the density function of Gaussian distribution in $\mathbb{R}^p$ with the mean vector being $\bm{\mu}_l$ and covariance matrix being $\sigma^2\bm{I}_p$.
Moreover, given $\bm{x}_1,\dots,\bm{x}_n$, the density function of each missingness indicator $\bm{r}_i$ can be written as 
\begin{align*}
    f(\bm{r}_i|\bm{x}_1,\dots,\bm{x}_n)
    = \prod_{j=1}^{p} q_{ij}^{r_{ij}}\cdot  \left( 1-q_{ij}\right)^{1-r_{ij}}, 
\end{align*}
where $q_{ij}$ ($i=1,\dots,n$, $j=1,\dots,p$) is the observed probability of $x_{ij}$. Under MNAR mechanisms, each $q_{ij}$ is a function of $x_{ij}$.
Thus, the joint density function of $\bm{x}_1,\dots,\bm{x}_n$ and $\bm{r}_1,\dots,\bm{r}_n$ is given by 
\begin{align*}
    &f(\bm{x}_1,\dots,\bm{x}_n,\bm{r}_1,\dots,\bm{r}_n;\bm{U},\bm{M}) \notag 
    = \prod_{i=1}^{n} f(\bm{x}_i ; \bm{U}, \bm{M}) \cdot f(\bm{r}_i|\bm{x}_1,\dots,\bm{x}_n).
\end{align*}
Write $\bm{X}=(\bm{x}_1^T,\dots,\bm{x}_n^T)^T$ and $\bm{R}=(\bm{r}_1^T,\dots,\bm{r}_n^T)^T$. Denote by $\bm{X}^{o}$ and $\bm{X}^{m}$ the observed and missing elements of $\bm{X}$, receptively. Then, to estimate $(\bm{U},\bm{M})$, we write the log-likelihood function of complete data $(\bm{X}^{o},\bm{X}^{m},\bm{R})$:
\begin{align*}
    &\ell_{n}(\bm{U},\bm{M};\bm{X}^{o},\bm{X}^{m},\bm{R}) \\
    &=\sum_{i=1}^{n} \log \big(  f(\bm{x}_i ; \bm{U},\bm{M}) \cdot f(\bm{r}_i|\bm{x}_1,\dots,\bm{x}_n)  \big)  \\
    &=\sum_{i=1}^{n} \sum_{l=1}^{k} u_{il}\cdot\left\{ - \frac{1}{2\sigma^2} \sum_{j=1}^{p} (x_{ij} - \mu_{lj})^2 \right\}  + \sum_{i=1}^{n} \sum_{j:r_{ij}=0} \log(1-q_{ij}) 
     -\frac{np}{2}\log(2\pi \sigma^2) + \sum_{i=1}^{n} \sum_{j:r_{ij}=1} \log (q_{ij}) ,
\end{align*} 
where the last two terms are constants when given $\bm{X}^{o},\bm{R}$. 
Suppose $2\sigma^2=1$, then $\max_{\bm{U},\bm{M}}\; \ell_{n}(\bm{U},\bm{M};\bm{X}^{o},\bm{X}^{m},\bm{R})$ is equivalent to: 
\begin{align}
    \min_{\bm{U},\bm{M}}\; \sum_{i=1}^{n} \sum_{l=1}^{k} u_{il}\cdot\left\{ \sum_{j=1}^{p} (x_{ij} - \mu_{lj})^2 \right\}  
    +  \sum_{i=1}^{n} \sum_{j:r_{ij}=0} \log\left(\frac{1}{1-q_{ij}} \right) .
    \label{eq_interpretation_min_problem_2} 
\end{align}
In Eq.(\ref{eq_interpretation_min_problem_2}), the first term is exactly the loss function of $k$-means on complete data, and the second term is the negative logarithm of the missing probability of all missing entries. 

Under the MCAR mechanism, since the missing probability is unrelated to $\bm{x}_1,\dots,\bm{x}_n$, then we have each $q_{ij}\in (0,1)$ is a constant. Consequently, the term $\sum_{i=1}^{n} \sum_{j:r_{ij}=0} \log\big(1/(1-q_{ij}) \big)$ is a constant and can be ignored when solving Eq.(\ref{eq_interpretation_min_problem_2}). 
Thereby, with the constrain $x_{ij}=\sum_{l=1}^{k} u_{il}\mu_{lj}$ for all $(i,j)$ satisfying $r_{ij}=0$, the minimization problem Eq.(\ref{eq_interpretation_min_problem_2}) is equivalent to 
\begin{align}
    &\min_{\bm{U},\bm{M}}\; \|\widehat{\bm{X}} - \bm{UM} \|_F^2  \notag\\
    & \text{s.t. } \widehat{\bm{X}}= \mathcal{P}_{\Omega}(\bm{X}) +  \mathcal{P}_{\Omega^c}(\bm{UM}),
\end{align}
which is exactly the objective function of existing $k$-~POD method. 
Unfortunately, since under MNAR mechanisms, the missing probability $1-q_{ij}$ is determined by the missing value $x_{ij}$ that is unknown, then the second term of Eq.(\ref{eq_interpretation_min_problem_2}) can not be ignored for optimization. 
The above probabilistic model shows that the existing $k$-POD method is essentially designed only for the MCAR mechanism, and thus no penalty term is needed. 
Whereas, for MNAR mechanisms, the $k$-POD ignores information of missing distribution, which causes the bias of estimation.

\subsection{Proposed method}
\label{subsec_method}

According to the analysis in Section~\ref{sec_why_kpod_fails}, we can see the key role of modeling and maximizing missing probability $1-q_{ij}$ of missing entries when clustering incomplete data under MNAR mechanisms. 
In general, estimating $q_{ij}$ is not identifiable without additional assumptions about missing mechanisms. 
Through this paper, we focus on a common scenario of MNAR mechanisms with magnitude-decaying missing probability, where data is more likely to be missing at positions with smaller absolute values. 
We can model such a scenario by the following formula, which is widely used for the single-cell RNA sequence data (scRNAseq data), where low-expression genes are more likely to have missing values than high-expression genes due to insufficient sequencing depth \cite{park2019sparse}. 
\begin{assumption}[Decaying squared exponential missingness]
    Given the original fully observed data, the observed probability of each $x_{ij}$ ($i=1,\dots,n$, $j=1,\dots,p$) is given by
    \begin{align}
    \text{Pr}\left( x_{ij}\text{ is observed } \;|\; \bm{x}_1,\dots,\bm{x}_n \right)= 1-\exp(-\lambda^{\ast} x_{ij}^2),
\end{align}
where $\lambda^{\ast} \geq 0$ is a hyper-parameter indicating the strength of missingness. 
\label{assumption_def_MNAR0}
\end{assumption}
Accordingly, we have $q_{ij}=1- \exp(-\lambda^{\ast} x_{ij}^2)$ for any $i=1,\dots,n$, $j=1,\dots,p$, then $\log\left(1/(1-q_{ij}) \right)= \lambda^{\ast} x_{ij}^2$.
It follows that with the constraint $x_{ij}=\sum_{l=1}^{k} u_{il}\mu_{lj}$ for all $(i,j)$ satisfying $r_{ij}=0$, the second term of Eq.(\ref{eq_interpretation_min_problem_2}) is proportional to $\|\mathcal{P}_{\Omega^c}(\bm{UM}) \|_F^2$. 
This suggests us to introduce the term $\|\mathcal{P}_{\Omega^c}(\bm{UM}) \|_F^2$ into clustering procedure. 

To this end, we propose our method as a minimization problem with respect to $\bm{U}\in\{0,1\}^{n\times k},\;\bm{U}\bm{1}_k=\bm{1}_n$ and $\bm{M}\in\mathbb{R}^{k\times p}$: 
\begin{align}
    &\min_{\bm{\bm{U},\bm{M}}}\; \|\widehat{\bm{X}} -\bm{UM}\|_F^2 + \lambda \| \mathcal{P}_{\Omega^c}(\bm{UM})\|_F^2 \notag\\
    &\text{s.t. }  \widehat{\bm{X}}=\mathcal{P}_{\Omega}(\bm{X})+\mathcal{P}_{\Omega^c}(\bm{UM}) 
    \label{eq_def_proposed}.
\end{align}
The first term is equivalent to the loss of $k$-means on the imputed data matrix $\widehat{\bm{X}}$. The second term constrains the imputation values at missing positions, to make the imputed data matrix $\widehat{\bm{X}}$ more close to the ground truth. 
The $\lambda\geq 0$ is a tuning parameter to control the strength of penalization.

The proposed method enjoys good interpretations. 
First, mathematically, since under Assumption~\ref{assumption_def_MNAR0}, the penalty term corresponds to the sum of the negative logarithm of the missing probability of all input missing entries, then minimizing such a penalty term is equivalent to maximizing these missing probabilities. 
This is necessary because under MNAR mechanisms, whether $x_{ij}$ is missing is also a random event, the occurrence probability of which is determined by the original true data values. Consequently, a statistically plausible imputation should yield a high theoretical missing probability, thereby aligning with the missing distribution reflected by the obtained incomplete dataset. 
Second, intuitively, the penalty term acts as a regularizer that directly controls the magnitude of the imputed values. In our considered MNAR scenario, if the $(i,j)$-th position is missing, then the true value of $x_{ij}$ is more likely to be very close to zero.  
However, traditional $k$-POD clustering fills in missing positions simply by the assigned cluster centers without considering the non-uniform missing distribution, which usually leads to over-imputation for the true data value. 
In contrast, our method uses a penalty to push large imputed values toward zero, reducing the imputation bias. 
As a result, our method yields a more accurate estimation of cluster centers and better performance of clustering.

\begin{remark}
We finally give some remarks for the proposed method. 
Compared with related works, the proposed penalty term is novel and non-trivial in the following sense. 
First, existing penalization-based $k$-means (e.g.: regularized, sparse, weighted and constrained $k$-means) use various norms ($\ell_1/\ell_2/\ell_0/$entropy) of centers $\bm{M}$ or feature weights or assignment $\bm{U}$ as penalty terms, mainly for dealing with high-dimensionality, local minima or unbalanced clusters \cite{sun2012regularized,Levrard2018,chang2018sparse,chakraborty2020entropy,raymaekers2022regularized}, whereas, these works focus on the case of complete data, and do not use formula of $\bm{UM}$ or quantities involving missing positions. 
Second, in the field of matrix completion, where the main purpose is to recover the underlying information of the obtained incomplete data matrix, many works consider constraining the approximation complete matrix $\bm{Y}$ by using the spectral norm, operator norm or the rank of $\bm{Y}$ as the penalty term \cite{candes2012exact,mazumder2010spectral,jin2022sparse}. Hastie et al. \cite{hastie2015matrix} uses the Frobenius norm of factorization matrices of $\bm{Y}$, instead of $\|\bm{Y}\|_F$. 
Therefore, the proposed penalty term differs from the existing works and plays an important role in improving the clustering result for MNAR data.

In addition, it should be noted that although the considered scenario anchors missingness at zero, the framework of our method can also work in applications of a non-zero reference point. Specifically, if high missingness is concentrated at a known point $\bm{c}\in\mathbb{R}^p$, we can model such a scenario by substituting the penalty with $\|\mathcal{P}_{\Omega}(\bm{UM}-\bm{1}_p\bm{c})\|_F^2$.    
\end{remark}

\section{Theoretical analysis}
\label{sec_theoretical_analysis}

In this section, we provide a theoretical analysis for the proposed method, and all proofs are provided in the Supplementary materials. 
Let $\bm{\mathrm{x}}_1,\dots,\bm{\mathrm{x}}_n$ be $n$ independent and identically distributed (i.i.d.) random vectors in $\mathcal{X}\subset \mathbb{R}^p$, where each $\bm{\mathrm{x}}_i=(\mathrm{x}_{i1},\dots,\mathrm{x}_{ip})$. 
Let $\mathbb{P}$ be a probability measure supported on $\mathcal{X}$, and $\mathbb{P}_n$ be the empirical measure based on $\bm{\mathrm{x}}_1,\dots,\bm{\mathrm{x}}_n$. 
For the MNAR mechanism given in Assumption~\ref{assumption_def_MNAR0}, we let $\bm{\mathrm{r}}_1,\dots,\bm{\mathrm{r}}_n$ be $n$ independent random vectors in $\{0,1\}^p$, where each $\bm{\mathrm{r}}_i=(\mathrm{r}_{i1},\dots,\mathrm{r}_{ip})$, and $\text{Pr}(\mathrm{r}_{i1}=1\;|\; \bm{\mathrm{x}}_1,\dots,\bm{\mathrm{x}}_n) = 1- \exp(-\lambda^{\ast}\mathrm{x}_{ij}^2)$. 
For simplification, we write $\bm{\mathrm{X}}=(\bm{\mathrm{x}}_1^T,\dots,\bm{\mathrm{x}}_n^T)^T$ and $\bm{\mathrm{R}}=(\bm{\mathrm{r}}_1^T,\dots,\bm{\mathrm{r}}_n^T)^T$. 
Then, based on the definition of cluster assignment $\bm{U}$ (i.e., $\bm{U}\bm{1}_k=\bm{1}_n$), we can write the empirical loss function of proposed method on the incomplete sample $\{(\bm{\mathrm{x}}_1,\bm{\mathrm{r}}_1),\dots,(\bm{\mathrm{x}}_n,\bm{\mathrm{r}}_n)\}$ as the function with respect to $\bm{M}$, i.e., 
\begin{align*}
    \widehat{L}_n^{\lambda}(\bm{M})
    =\frac{1}{n}\sum_{i=1}^{n} \min_{l=1,\dots,k} \sum_{j=1}^{p} \mathrm{r}_{ij} ( \mathrm{x}_{ij} - \mu_{lj} )^2 + (1-\mathrm{r}_{ij})\lambda\mu_{lj}^2. 
\end{align*}
Moreover, we can define the estimated cluster centers by the unconstrained minimizer of $\widehat{L}_n^{\lambda}$, i.e.,
\begin{align*}
    \widehat{\bm{\mathrm{M}}}^{\lambda} = \mathop{\arg\min}_{\bm{M}\in\mathcal{M}} \widehat{L}_n^{\lambda}(\bm{M}),
\end{align*}
where the minimizer is supposed to be unique. In addition, the set $\mathcal{M}=\{\bm{M}\;|\; \bm{\mu}_l\in\mathcal{X}, \|\bm{\mu}_l\|_2\leq B,\forall l=1,\dots,k\}$, where $B>0$ is a pre-given constant. 
Accordingly, the expected counterpart of $\widehat{L}_n^{\lambda}$ is given by 
\begin{align*}
    L^{\lambda}(\bm{M})= \mathbb{E}_{\bm{\mathrm{x}}_1,\bm{\mathrm{r}}_1} \left[  \min_{l=1,\dots,k} \sum_{j=1}^{p} \mathrm{r}_{1j} ( \mathrm{x}_{1j} - \mu_{lj} )^2 + (1-\mathrm{r}_{1j})\lambda\mu_{lj}^2  \right],
\end{align*}
and we let 
\begin{align*}
    \bm{M}^{\ast,\lambda} =\mathop{\arg\min}_{\bm{M}\in\mathcal{M}} L^{\lambda}(\bm{M}),
\end{align*}
where the minimizer is supposed to be unique.

Our first theoretical result gives the consistency of the estimated cluster centers. 
\begin{assumption}
\label{assumption_X_compact_sub-Gaussian}
    Suppose that the data space $\mathcal{X}$ is compact and $\bm{\mathrm{x}}_1$ is sub-Gaussian. 
\end{assumption}
\begin{theorem}
\label{theorem_consistency}
    Under the above settings and Assumption~\ref{assumption_X_compact_sub-Gaussian}, we have for any $\epsilon>0$, $\lim_{n\rightarrow\infty} \textnormal{Pr}(\| \text{vec}( \widehat{\bm{\mathrm{M}}}^{\lambda}) - \text{vec}(\bm{M}^{\ast,\lambda} ) \|_2 >\epsilon)=0$. 
\end{theorem}

The above theorem shows that given a fixed $\lambda$, the estimated cluster centers $\widehat{\bm{\mathrm{M}}}^{\lambda}$ would converge to the minimizer of the expected counterpart of loss function as $n\rightarrow\infty$, where the required Assumption~\ref{assumption_X_compact_sub-Gaussian} is a commonly used setting for deriving the uniform convergence of $\widehat{L}_n^{\lambda}(\cdot)$ to $L^{\lambda}(\cdot)$. 
It should be noted that the limit $\bm{M}^{\ast,\lambda}$ varies with different $\lambda$, while the ``true cluster centers" only rely on the distribution of original fully observed data $\bm{\mathrm{x}}_i$'s.

Our second theoretical result thus gives some conditions under which the estimation of the proposed method on incomplete data could recover the ground truth of cluster structure of the original fully observed data. 
Let $\mathrm{z}_i$ be the random variable denoting which cluster $\bm{\mathrm{x}}_i$ belongs to, and $\text{Pr}(\mathrm{z}_i=l)=1/k$ for any $l=1,\dots,k$. 
Suppose that given $\mathrm{z}_i=l$, the conditional distribution of $\bm{\mathrm{x}}_i$ is a normal distribution $\mathcal{N}(\bm{\mu}_l^{\ast\ast},\sigma^2\bm{I}_p)$. 

\begin{assumption}
\label{assumption_Mstarstar_is_minimizer}
    Suppose that $\bm{\mu}_l^{\ast\ast}$'s are well-separated with each other relative to $\sigma^2$ such that $\bm{M}^{\ast\ast}=((\bm{\mu}_1^{\ast\ast})^T,\dots,(\bm{\mu}_k^{\ast\ast})^T)^T$ uniquely minimizes the $k$-means on fully observed data in the population level, i.e.,
    \begin{align*}
        \bm{M}^{\ast\ast} = \mathop{\arg\min}_{\bm{M}\in\mathcal{M}} \mathbb{E}_{\bm{\mathrm{x}}_1} \left[ \min_{l=1,\dots,k}\| \bm{\mathrm{x}}_1 - \bm{\mu}_l \|_2^2 \right],
    \end{align*}
    where we refer $\bm{M}^{\ast\ast}$ as ``true cluster centers".
\end{assumption}

\begin{theorem}
\label{theorem_converge_to_truth}
    Under Assumptions~\ref{assumption_def_MNAR0}-\ref{assumption_Mstarstar_is_minimizer} and suppose that $\bm{M}^{\ast\ast}$ satisfies $\mu_{lj}^{\ast\ast}\neq \mu_{l'j}^{\ast\ast}$ for any $l\neq l'$ and $j=1,\dots,p$. If $\lambda=1-1/(2\sigma^2\lambda^{\ast}+1)$ and $\sigma^2\log(n)\rightarrow 0$ as $n\rightarrow\infty$, then we have the estimated cluster centers $\widehat{\bm{\mathrm{M}}}^{\lambda}$ would converge to the true cluster centers $\bm{M}^{\ast\ast}$ in probability as $n\rightarrow\infty$. 
\end{theorem}

The above theorem shows that in the high signal-to-noise ratio (SNR) regime $\sigma^2\log(n)\rightarrow 0$, the estimated result of the proposed method can asymptotically recover the ground truth of cluster structure by taking $\lambda=1-1/(2\sigma^2\lambda^{\ast}+1)$. 
The high SNR condition is to ensure a low mis-clustering rate, and such $\lambda$ can ensure the successful recovery of true centers $\bm{M}^{\ast\ast}$ given a true partition of the sample. 
Moreover, the above condition of $\lambda$ includes a scaling factor $(2\sigma^2\lambda^{\ast}+1)$, which reflects the distortion of the data distribution due to the MNAR missingness (Assumption~\ref{assumption_def_MNAR0}). 
Specifically, given $\mathrm{z}_i=l$, under our MNAR mechanism, the missing part of data (i.e., $\mathrm{x}_{ij}|\mathrm{r}_{ij}=0$) actually follows a normal distribution with mean $\mu_{lj}^{\ast\ast}/(2\sigma^2\lambda^{\ast}+1)$ and variance $\sigma^2/(2\sigma^2\lambda^{\ast}+1)$. 
Compared with distribution of full observed data (i.e., $\mathcal{N}(\mu_{lj}^{\ast\ast},\sigma^2)$), it can be seen that the factor $(2\sigma^2\lambda^{\ast}+1)$ characterizes the shrinkage of the distribution. 
As a consequence, the choice of $\lambda=1-1/(2\sigma^2\lambda^{\ast}+1)$ enables the penalty term $(1-\mathrm{r}_{ij})\lambda \mu_{lj}^2$ to calibrate the estimation bias of cluster centers caused by ignoring the missing part of data.

\section{Optimization}
\label{sec_optimization}

\subsection{Algorithm}
\label{subsec_algorithm}

We note that the proposed method is equivalent to an unconstrained problem with respect to membership matrix $\bm{U}\in\{0,1\}^{n\times k}$ with $\bm{U}\bm{1}_k=\bm{1}_n$ and cluster centers $\bm{M}\in\mathbb{R}^{k\times p}$: 
\begin{align}
    \min_{\bm{U},\bm{M}} \|\mathcal{P}_{\Omega}(\bm{X}-\bm{UM})\|_F^2 + \lambda\| \mathcal{P}_{\Omega^c}(\bm{UM}) \|_F^2,
\end{align}
where the loss function can be expanded to be  
\begin{align}
    \widehat{L}_n^{\lambda}(\bm{U},\bm{M})=\sum_{i=1}^{n}\sum_{l=1}^{k} u_{il} \left\{  \sum_{j:r_{ij}=1} (x_{ij}-\mu_{lj})^2 + \lambda \sum_{j:r_{ij}=0} \mu_{lj}^2 \right\}, 
    \label{eq_pkpod_loss_scalar}
\end{align}
where $\bm{R}\in\{0,1\}^{n\times p}$ is the missing indicator matrix satisfying $r_{ij}=1$ if $x_{ij}$ is observed, 0 otherwise. 
Then, our aim is to minimize $\widehat{L}_n^{\lambda}(\bm{U},\bm{M})$ with respect to $\bm{U}$ and $\bm{M}$.  

Given $\bm{M}$, since the membership matrix $\bm{U}$ is binary, then the $\widehat{L}_n^{\lambda}$ is minimized when each data point is assigned to the cluster center satisfying: it is close to the data point in observed features and close to zero in missing features. 
Given $\bm{U}$, since the $\widehat{L}_n^{\lambda}$ is convex with respect to $\bm{M}$, then it is minimized when the partial derivatives are equal to zero. 
Therefore, we can solve the minimization problem by alternatively minimizing $\widehat{L}_n^{\lambda}(\bm{U},\bm{M})$ with respect to $\bm{U}$ and $\bm{M}$, for which we propose Algorithm~\ref{algorithm_altmin}.

\begin{algorithm}[tb]
    \caption{MNAR-$k$-means algorithm}
    \label{algorithm_altmin}
    \textbf{Input}: Incomplete data matrix $\bm{X}$, Missing indicator $\bm{R}$ \\
    \textbf{Parameter}: Tuning parameter $\lambda$, maximum number of iterations $t_{\max}$\\
    \textbf{Output}: Estimator $\widehat{\bm{U}}$ and  $\widehat{\bm{M}}$
    \begin{algorithmic}[1] 
        \State Initialize $\bm{M}^{(0)}$, set $t=0$.
        \While{$t\leq t_{\max}$ and Eq.(\ref{eq_pkpod_loss_scalar}) does not converge}
        \State Update $\bm{U}^{(t+1)}$ by $u_{il}^{(t+1)}=\mathds{1}(l=l_i^{\ast})$ for any $i=1,\dots,n$, $l=1,\dots,k$, where
        \[
            l_i^{\ast}= \mathop{\arg\min}_{l=1,\dots,k} \sum_{j:r_{ij}=1}\left(x_{ij} - \mu_{lj}^{(t)}\right)^2 + \lambda \sum_{j:r_{ij}=0} \left(\mu_{lj}^{(t)}\right)^2. 
        \]
        \State Update $\bm{M}^{(t+1)}$ by for any $l=1,\dots,k, j=1,\dots,p$, 
        \[
            \mu_{lj}^{(t+1)}=\frac{\sum_{i: r_{ij}=1} u_{il}^{(t+1)} x_{ij}  }{  \sum_{i: r_{ij}=1} u_{il}^{(t+1)} + \lambda\sum_{i:r_{ij}=0} u_{il}^{(t+1)}  }. 
        \]
        \State $t \leftarrow t+1$.
        \EndWhile
        \State \textbf{return} $\bm{U}^{(t+1)}$, $\bm{M}^{(t+1)}$. 
    \end{algorithmic}
\end{algorithm}

\subsection{Convergence and complexity}
\label{subsec_convergence_complexity}

The convergence of Algorithm~\ref{algorithm_altmin} can be guaranteed in the sense that the value of loss function of each iteration converges to a finite limit corresponding to a local minimum. This follows from the fact that both update steps for $\bm{U}^{(t+1)}$ (Line 3) and $\bm{M}^{(t+1)}$ (Line 4) are exact minimization steps and thus the sequence $\left\{ \widehat{L}_n^{\lambda}(\bm{U}^{(t)}, \bm{M}^{(t)}) \right\}_{t=0}^{\infty}$ is non-increasing as $t$ increases, and moreover, $\widehat{L}_n^{\lambda}\geq 0$ for any $(\bm{U},\bm{M})$. In practice, Algorithm~\ref{algorithm_altmin} usually converges within a small number of iterations. 

The computational complexity of Algorithm~\ref{algorithm_altmin} is $O(t_{\max} npk)$, because updating both $\bm{U}$ and $\bm{M}$ need $O(npk)$ operations, and the loop with respect to $t$ stops within $t_{\max}$ iterations.

\subsection{Initialization}
\label{subsec_initialization}

In Algorithm~\ref{algorithm_altmin}, we initialize cluster centers matrix $\bm{M}^{(0)}$ as follows. 
First, we impute missing values of $\bm{X}$ by corresponding columns' means of observed values, to obtain a new complete data matrix $\bm{X}^{(0)}\in\mathbb{R}^{n\times p}$. 
Second, we apply the $k$-means++ method proposed by \cite{arthur2007k} on the imputed data matrix $\bm{X}^{(0)}$ to get $k$ centers, being $k$ rows of $\bm{M}^{(0)}$. 

Moreover, to overcome the local minimum, multiple initializations are used independently, and we select the result with the smallest value of $\widehat{L}_n^{\lambda}$ to be the output.

\subsection{Selection of tuning parameter}
\label{subsec_tuningparameter}

We use the instability index as the criterion for tuning parameter $\lambda$, which was first proposed by \cite{wang2010consistent} and has been widely used for tuning parameters in the field of clustering. 
The main idea is that a good value for the tuning parameter should yield a stable clustering in response to minor disruption to the sample. 

In practice, the instability index is calculated as follows. (1) The input dataset with sample size $n$ is randomly divided into three subsets, two of which consist of $m$ data points as training sets and the third one as a validation set. (2) Conduct the proposed clustering method with some $\lambda$ on both training sets separately and obtain two estimation results of cluster centers. (3) Based on the two estimation results, predict cluster labels for data points in the validation set. (4) Calculate the disagreement between two prediction results. 
Repeat the procedure several times, and the instability index for the parameter $\lambda$ is given by the average disagreement. 

\section{Simulations}
\label{sec_simulations}

In this section, we validate the performance of the proposed method on synthetic datasets. 
We evaluate both the accuracy of clustering and the estimation of cluster centers. 
For the accuracy of clustering, we use the Classification Error Rate (CER) as the evaluation index, that is, $\text{CER}(\widehat{\mathcal{C}}, \mathcal{C}^{\ast\ast}) = 2\left\{n(n-1)\right\}^{-1} \sum_{i>i'} \left| \mathds{1}_{ \widehat{\mathcal{C}}(i,i') } -  \mathds{1}_{ \mathcal{C}^{\ast\ast}(i,i') } \right|$. Here, $\widehat{\mathcal{C}}$ is the estimated partition of the dataset given by $\widehat{\bm{U}}$, and $\mathcal{C}^{\ast\ast}$ is the true partition of the dataset, and $\mathds{1}_{ \mathcal{C}(i,i') } =1$ if the $i$-th and $i'$-th data points belong to the same cluster according to the partition $\mathcal{C}$, 0 otherwise. 
For the estimation of cluster centers, we use the mean-squared error (MSE) of the estimated cluster centers as the evaluation index, that is, $\text{MSE}(\widehat{\bm{M}},\bm{M}^{\ast\ast})=\sum_{l=1}^{k}\min_{l'=1,\dots,k} \|\hat{\bm{\mu}}_l - \bm{\mu}^{\ast\ast}_{l'}\|^2$, where $\widehat{\bm{M}}$ is the estimated cluster centers, and $\bm{M}^{\ast\ast}=((\bm{\mu}_1^{\ast\ast})^T,\dots,(\bm{\mu}_k^{\ast\ast})^T)^T$ is the true cluster centers. 
All codes for reproduction can be found in \url{https://github.com/GXguanxin/MNAR-k-means}. 

\subsection{Effect under Assumption~\ref{assumption_def_MNAR0}}
\label{subsec_experiment_MNAR0}

We evaluate the effect of improving the existing $k$-POD for missing data under the MNAR mechanism satisfying Assumption~\ref{assumption_def_MNAR0} with different values of $\lambda^{\ast}$, corresponding to different missing proportions. 

To construct such a missing dataset, we first generate a dataset containing $n$ fully observed data points $\bm{x}_i \in \mathbb{R}^2$ ($i=1,\dots,n$), which consists of $k$ clusters. For any $l=1,\dots,k$, the probability of each $\bm{x}_i$ belonging to $l$-th cluster is equal to $1/k$, and these $\bm{x}_i$'s in $l$-th cluster are drawn independently from a Gaussian distribution $\mathcal{N}(\bm{\mu}_l^{\ast\ast},\Sigma)$ with $\Sigma=\text{diag}(1,\dots,1)$.  
The following settings of true cluster centers are considered in this section: 
\begin{itemize}
    \item $k=2$: $\bm{\mu}_1^{\ast\ast}=(0,0)$, $\bm{\mu}_2^{\ast\ast}=(a,0)$;
    \item $k=3$: $\bm{\mu}_1^{\ast\ast}=(0,0)$, $\bm{\mu}_2^{\ast\ast}=(a,0)$, $\bm{\mu}_3^{\ast\ast}=(a/2,\sqrt{3}a/2)$;
    \item $k=4$: $\bm{\mu}_1^{\ast\ast}=(0,0)$, $\bm{\mu}_2^{\ast\ast}=(a,0)$, $\bm{\mu}_3^{\ast\ast}=(a,0)$, $\bm{\mu}_4^{\ast\ast}=(a,a)$,
\end{itemize}
where $a\in \{2,3,4\}$ are considered to express different separations between true cluster centers. 
Secondly, for each $x_{ij}$ ($i=1,\dots,n$, $j=1,2$), we artificially let it be missing by the following mechanism: $\text{Pr}\left( x_{ij} \text{ is missing } \right) = \exp(-\lambda^{\ast} x_{ij}^2)$, where $\lambda^{\ast}\in \{0.1, 0.5, 1, 2, 10\}$ are considered to express different missing proportions from 50\% to 10\%. 
In addition, we let $n=100k$, which means each cluster has $100$ data points, and for implementation, we exclude those data points that all elements $x_{ij}$ are missing. 

Figure~\ref{fig_CER_MSE_MNAR0} illustrates how the results of CER and MSE of the proposed method vary with different $\lambda^{\ast}$ in different settings of $a$ and $k$. The averaged values of 30 repetitions are reported, as well as standard deviations. As the benchmark, the black dotted line is the result of $k$-means on the original fully observed dataset. 
We can see that in almost all settings, the CER and MSE values of the proposed method become smaller as $\lambda^{\ast}$ becomes larger, which corresponds to less missingness in the dataset. 
Moreover, compared with $k$-POD method, the proposed method has smaller CER and MSE values in almost all settings, showing a significant improvement in the accuracy of clustering and estimated cluster centers. 
Furthermore, in most settings ($\lambda^{\ast}\geq 0.5$, i.e, missing proportion less than 40\%), the result of the proposed method is very close to that of $k$-means on the original fully observed dataset, showing its ability to recover the clustering performance of the $k$-means method even under MNAR mechanisms. 

\begin{figure}[!t]
    \centering
    \includegraphics[width=0.33\linewidth]{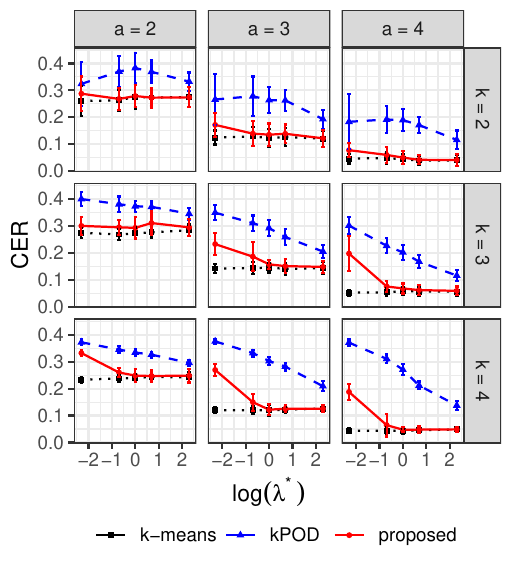}
    \includegraphics[width=0.33\linewidth]{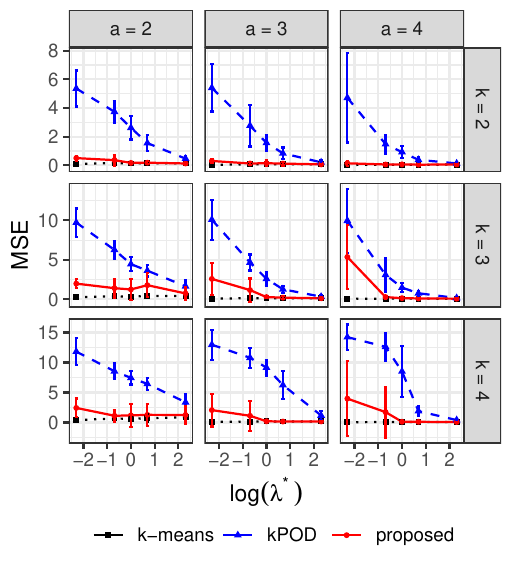}
    \caption{The results of CER and MSE of the proposed method under the MNAR mechanism satisfying Assumption~\ref{assumption_def_MNAR0}. The $\lambda^{\ast}\in \{0.1, 0.5, 1, 2, 10\}$ express different missing proportions from 50\% to 10\%. The black dotted line is the result of $k$-means on the original fully observed dataset. Blue dashed line and red solid line are for $k$-POD and proposed method, respectively. }
    \label{fig_CER_MSE_MNAR0}
\end{figure}

\subsection{Comparison with other methods}
\label{subsec_experiment_comparison}

We compare the proposed method with other methods for data under different missing mechanisms and with different missing proportions. 

The missing data is generated in a similar way to that in Section~\ref{subsec_experiment_MNAR0}. 
The following settings of true cluster centers in $\mathbb{R}^{50}$ are considered: 
\begin{itemize}
    \item $k=3$: $\bm{\mu}_1^{\ast\ast}=(0,0,0,\dots,0)$, $\bm{\mu}_2^{\ast\ast}=(3,0,0,\dots,0)$, $\bm{\mu}_3^{\ast\ast}=(3/2,3\sqrt{3}/2,0,\dots,0)$;
    \item $k=4$: $\bm{\mu}_1^{\ast\ast}=(0,0,0,\dots,0)$, $\bm{\mu}_2^{\ast\ast}=(3,0,0,\dots,0)$, $\bm{\mu}_3^{\ast\ast}=(0,3,0,\dots,0)$, $\bm{\mu}_4^{\ast\ast}=(3,3,0,\dots,0)$,
\end{itemize}
which means that only the first two features of data are relevant to the true cluster structure, while the other 48 features are irrelevant features. 
Moreover, we consider general missing mechanisms that the $(i,j)$-th position is more likely to be missing if the value of $x_{ij}$ is more close to zero. 
The following commonly used missing mechanisms are considered:  
\begin{itemize}
    \item MNAR0: $\text{Pr}\left( x_{ij} \text{ is missing } \right) = \exp(-\lambda^{\ast} x_{ij}^2)$;
    \item MNAR1: $\text{Pr}\left( x_{ij} \text{ is missing } \right) = 1/\left(1+\exp(\phi^{\ast} x_{ij}^2)\right)$;
    \item MNAR2: For any $j=1,\dots,p$, let $x_{ij}$ be missing if its absolute value $|x_{ij}|$ is less than the bottom $\alpha^{\ast}$ quantile of $\{|x_{1j}|,\dots,|x_{nj}|\}$,
\end{itemize}
where hyper-parameters $\lambda^{\ast}>0$, $\phi^{\ast}>0$ and $\alpha^{\ast}\in (0,1)$ take different values to meet different missing proportions from 10\% to 50\%. 
We consider 6 peer methods as follows: 
\begin{itemize}
    \item $k$-means: We conduct $k$-means on original fully observed dataset as benchmark. The built-in R function {\fontfamily{qcr}\selectfont kmeans} with Lloyd's algorithm \cite{Lloyd1982} is used. 
    \item Mean imputation: We impute the missing entries by corresponding column means of observed values, and then conduct $k$-means on the imputed data matrix. 
    \item Multiple imputation: We impute the missing entries via the popular \textit{mice} model \cite{buuren2011mice}, where the R package {\fontfamily{qcr}\selectfont mice} is used to get five complete data matrices after imputation. We pool the five matrices by element-wise average into a single matrix and then conduct $k$-means. 
    \item KNN imputation: We impute the missing entries of each data point by the average of observed values of its $10$ nearest neighbors, where the R package {\fontfamily{qcr}\selectfont impute} \cite{hastie2021} is used. Then we conduct $k$-means on the imputed data matrix. 
    \item FWPD: We use a modified Euclidean distance \textit{FWPD} proposed by Datta et al. \cite{datta2018clustering} for clustering of missing data. We run the MATLAB code provided by the author to analyze the same simulation data generated in R. 
    \item The $k$-POD clustering. We use the proposed algorithm with a tuning parameter equal to zero as a substitute of the original R package {\fontfamily{qcr}\selectfont kpodclustr} \cite{chi2016k}. 
\end{itemize}
In the implementation of all methods, 100 initializations and a maximum of 100 iterations are used. For our method, the tuning parameter $\lambda$ is selected by the instability criterion from a set of candidates $\{0.1, 2, 4, 6, 8, 10, 12, 14, 16, 18, 20\}$.

Figure~\ref{fig_CER_MSE_different_MNAR} illustrates the results of CER and MSE of different methods in different settings of missing mechanisms and missing proportions. The box plots of 30 repetitions are reported. 
In each panel, 6 boxes are for $k$-means on the original fully observed dataset, mean imputation, multiple imputation via MICE, KNN imputation, $k$-POD clustering and the proposed method, respectively, where FWPD is excluded since it fails to give a reasonable clustering result for data over 30\% missingness and the averaged CER for 10\% missingness exceeds 0.4.
We can see that in almost all settings, the proposed method has the smallest CER and MSE values than other methods, showing its highest clustering accuracy and best performance of estimating cluster centers. 
Moreover, the better performance is more significant for data with a larger missing proportion, implying the better stability of the proposed method to cope with serious and complex missingness. 
In addition, for all settings, we conducted t-test on CER (MSE as well) of proposed method with each peer method, and all $p$-values are less than 0.001, which confirms the statistical significance of the better performance of proposed method.

\begin{figure}[!t]
    \centering
    \begin{minipage}{0.33\textwidth}
    \centering
        \includegraphics[width=\linewidth]{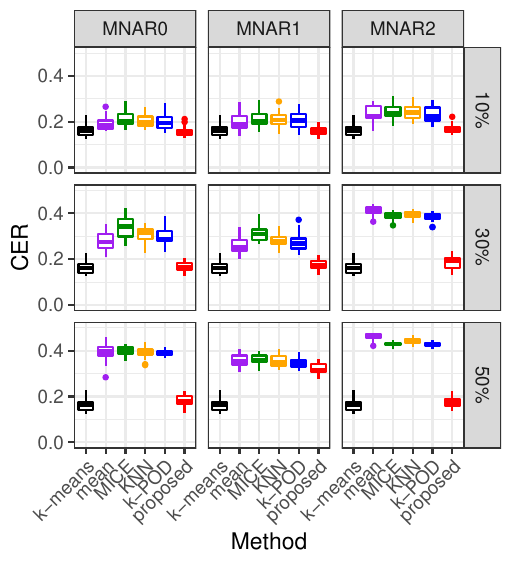}
        {\small $k=3$}
    \end{minipage}
    \begin{minipage}{0.33\textwidth}
    \centering
        \includegraphics[width=\linewidth]{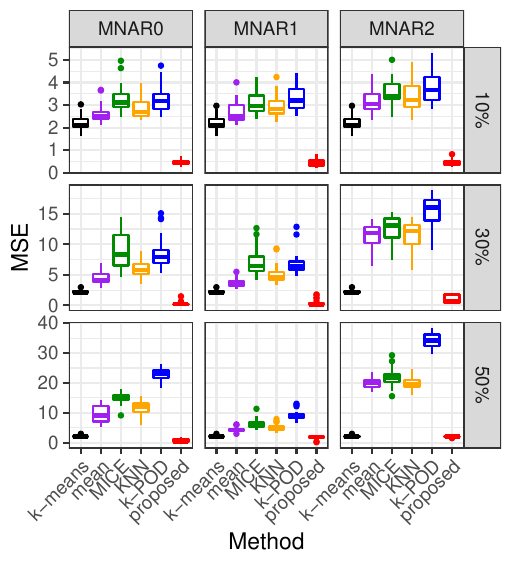}
        {\small $k=3$}
    \end{minipage}
    \par
    \par
    \begin{minipage}{0.33\textwidth}
    \centering
        \includegraphics[width=\linewidth]{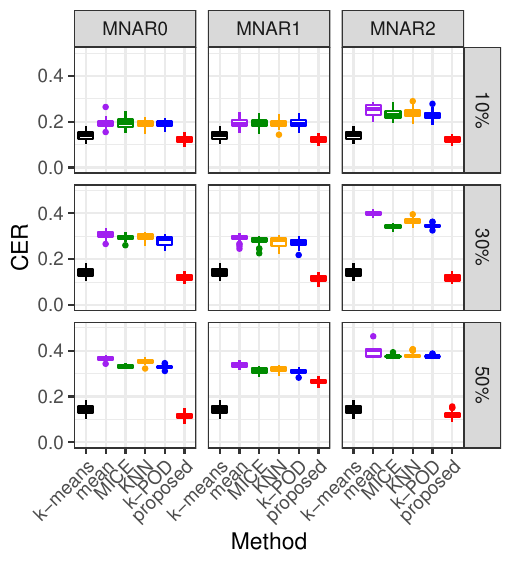}
        {\small $k=4$}
    \end{minipage}
    \begin{minipage}{0.33\textwidth}
    \centering
        \includegraphics[width=\linewidth]{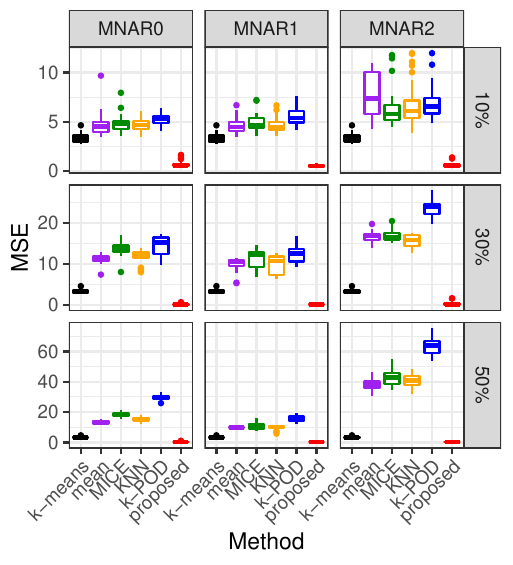}
        {\small $k=4$}
    \end{minipage}
    \caption{The results of CER and MSE of different methods under different missing mechanisms and missing proportions. The top row is for $k=3$, and the bottom row is for $k=4$. In each panel, 6 boxes are: $k$-means on the original fully observed dataset, mean imputation, multiple imputation via MICE, KNN imputation, $k$-POD and the proposed method, respectively. }
    \label{fig_CER_MSE_different_MNAR}
\end{figure}

\subsection{Analysis of Algorithm} 

\subsubsection{Algorithm convergence}
\label{subsec_experiment_convergence}

In Algorithm~\ref{algorithm_altmin}, we iteratively update membership $\bm{U}$ and cluster centers $\bm{M}$ until the loss function converges. As explained in Section~\ref{subsec_convergence_complexity}, the convergence of $\widehat{L}_n^{\lambda}$ is guaranteed. We here numerically verify that $\widehat{L}_n^{\lambda}$ converges within a small number of iterations. 
Table~\ref{table_loss_convergence} provides the number of iterations of the proposed algorithm, where we consider the fully observed data with $p=50$ and $k=3$ and different missingness. The averaged values of 30 repetitions are reported, as well as standard deviations.
It can be seen that the proposed algorithm stops after fewer than about 10 iterations in all settings, which implies the convergence of the loss function of the proposed algorithm. 

\begin{table}[tbhp]
\centering
\caption{The number of iterations of proposed algorithm (standard deviations in brackets) }
\label{table_loss_convergence}
\begin{adjustbox}{center}
\begin{tabular}{@{}cccc@{}}
\toprule
 \makecell{Missing proportion} & MNAR0 & MNAR1 & MNAR2\\
\midrule
10\% & 6.833 (4.70) & 5.867 (2.66) & 6.933 (3.16) \\
30\% & 5.733 (2.72) & 5.567 (2.50) & 6.267 (1.44) \\
50\% & 5.700 (1.93) & 6.367 (1.38) & 6.233 (1.61) \\
 \bottomrule
\end{tabular}
\end{adjustbox}
\end{table}

\subsubsection{Computation time}
\label{subsec_experiment_computation_time}

We here verify the linear computation time with respect to sample size $n$ and data dimension $p$ via numerical experiments, which is explained in Section~\ref{subsec_convergence_complexity}. 
We consider the fully observed data with $k=3$ and the MNAR0 missing mechanism with $\lambda^{\ast}=1$. 
First, we fix $p=10$ and vary the total sample size $n$ from 300 to 1500. Then, we fix $n=300$ and vary $p$ from 10 to 50. 
The averaged computation time of 30 repetitions is reported in Figure~\ref{fig_computation_time}. It can be seen that the computation time linearly increases as $n$ or $p$ becomes larger. 

\begin{center}
\begin{minipage}[t]{0.6\textwidth}
    \centering
    \includegraphics[height=5cm]{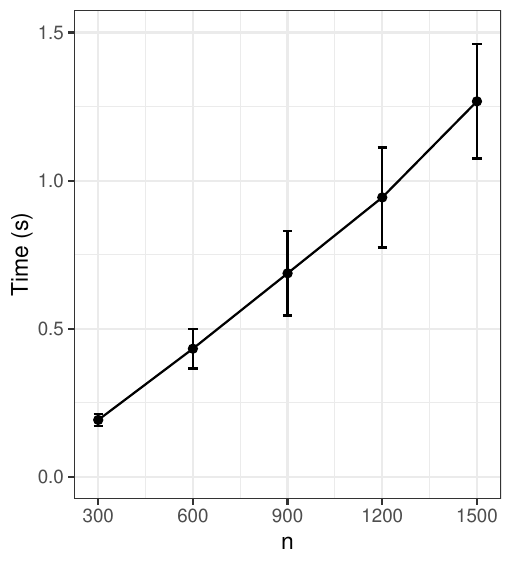}
    \includegraphics[height=5cm]{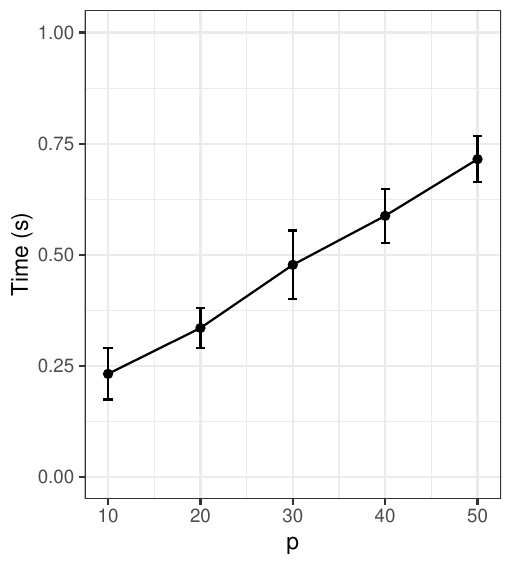}
    \captionsetup{type=figure,width=\linewidth}
    \captionof{figure}{The computation time of the proposed method on missing data with $k=3$ and under the MNAR0 mechanism with $\lambda^{\ast}=1$. \textbf{Left}: fix $p=100$ and vary $n$; \textbf{Right}: fix $n=300$ and vary $p$.}
    \label{fig_computation_time}
\end{minipage}
\hfill
\begin{minipage}[t]{0.35\textwidth}
    \centering
    \includegraphics[height=5cm]{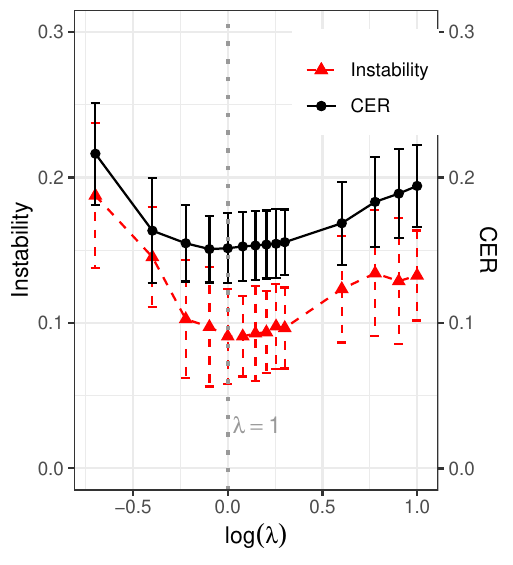}
    \captionsetup{type=figure,width=\linewidth}
    \captionof{figure}{The CER values of the proposed method with different values of $\lambda$, and corresponding instability indexes. The vertical gray dotted line is $\lambda=1$ selected by the instability criterion.}
    \label{fig_tuningparameter}
\end{minipage}
\end{center}

\subsubsection{Stability of initialization}

We here analyze the stability of different starting points $\bm{M}^{(0)}$ generated by the strategy in Section~\ref{subsec_initialization}. We calculate the fluctuation in the final loss function values among different starting points. We consider the case of $k=3$ with 30\% MNAR0 missingness, and fix $n=300$. 

Table~\ref{table_initialization_fluctuation} illustrates the indexes describing fluctuation of final loss values across 100 starting points, including the standard deviation ($\text{sd}$), the range between maximum and minimum ($\max - \min$) and the interquartile range between 75\% quantile and 25\% quantile ($\text{3rd Qu.} - \text{1st Qu.}$), where the averaged results of 30 repetitions are reported. For reference, we also provide the decrease from the corresponding initial loss. 
We can see that in all settings, the standard deviation and the interquartile range are both small when compared with the decrease from corresponding initial loss values, even though the range between maximum and minimum is relatively large. It means that while a few poor starting points may lead to undesirable local minima, the overall final loss values remain highly concentrated with a small standard deviation. 
This further verifies the stability of the proposed alternative minimization algorithm in consistently converging to high-quality local minima, and the effectiveness of the multiple initializations in avoiding poor solutions. 

Additionally, we explored alternative initialization strategies such as ``randomly drawing $k$ points from a column-imputed dataset" and ``using sparse approximations of $k$-POD's output as a starting point". Since these methods yielded performance comparable to the current strategy, we retained the \textit{kmeans++}-based initialization strategy in our methodology.

\begin{table}[tbhp]
\centering
\caption{Fluctuation of final loss values across 100 starting points}
\label{table_initialization_fluctuation}
\begin{adjustbox}{center}
\begin{tabular}{@{}lccccc@{}}
\toprule
\makecell{Missing\\mechanism} & \makecell{Missing\\proportion} & $\text{sd}$ & $\max - \min$ & $\text{3rd Qu.} - \text{1st Qu.}$ & \makecell{Decrease from\\ initial loss} \\
\midrule
MNAR0 & 10\% & 0.034 & 0.302 & 0.002 & 0.292 \\
 & 30\% & 0.088 & 0.648 & 0.001 & 0.907 \\
 & 50\% & 0.320 & 1.238 & 0.154 & 2.105 \\
MNAR1 & 10\% & 0.016 & 0.127 & 0.002 & 0.510 \\
 & 30\% & 0.043 & 0.319 & 0.001 & 0.597 \\
 & 50\% & 0.026 & 0.144 & 0.001 & 0.919 \\
MNAR2 & 10\% & 0.051 & 0.421 & 0.001 & 0.487 \\
 & 30\% & 0.282 & 1.145 & 0.300 & 2.197 \\
 & 50\% & 0.828 & 2.650 & 1.191 & 3.526 \\
 \bottomrule
\end{tabular}
\end{adjustbox}
\end{table}

\subsubsection{Sensitivity of tuning parameter}
\label{subsec_experiment_sensitivity}

We further analyze how the tuning parameter $\lambda$ influences the proposed method, and evaluate the effect of the instability criterion in selecting $\lambda$. 
As an example, we consider the setting of $n=100k$, $p=2$, $k=3$, and true cluster centers are given by $\bm{\mu}_1^{\ast}=(0,0)$, $\bm{\mu}_2^{\ast}=(a,0)$ and $\bm{\mu}_3^{\ast}=(a/2,\sqrt{3}a/2)$, and we consider missing mechanism MNAR0 with $\lambda^{\ast}=1$. 
Figure~\ref{fig_tuningparameter} illustrates CER values of the proposed method with 15 different values of $\lambda$, as well as corresponding instability indexes. The averaged values and standard deviations of 30 repetitions are reported. 
It can be seen that since $\lambda=1$ has the lowest instability, it is selected by the instability criterion, and it indeed has the smallest CER, which implies the utility of the instability criterion for selecting an appropriate tuning parameter. 

\subsection{Analysis of statistical consistency}

In this section, we verify the statistical consistency of estimated cluster centers of proposed method to true cluster centers under conditions given in Theorem~\ref{theorem_converge_to_truth}. 
We generate missing datasets in a similar way as in Section~\ref{subsec_experiment_MNAR0}, where we let $p=2$ and $k=3$ true cluster centers given by $\bm{\mu}_1^{\ast\ast}=\left(-\sqrt{6}/2, -\sqrt{6}/2\right)$, $\bm{\mu}_2^{\ast\ast}=\left( \left\{\sqrt{6}+3\sqrt{2}\right\}/4, \left\{\sqrt{6}-3\sqrt{2}\right\}/4 \right)$ and $\bm{\mu}_3^{\ast\ast}=\left( \left\{\sqrt{6}-3\sqrt{2}\right\}/4, \left\{\sqrt{6}+3\sqrt{2}\right\}/4 \right)$, 
which means that cluster centers are separated from each other in every dimension and satisfy the separation condition. 
We consider the missing mechanism satisfying Assumption~\ref{assumption_def_MNAR0} with $\lambda^{\ast}=1$, which leads to about 30\% missingness. 
Moreover, we let $n\in \{90 ,150, 300 , 900 , 1500, 2100, 3000 ,9000,15000, 30000\}$ and take $\Sigma=\text{diag}(\sigma^2,\sigma^2)$, where $\sigma^2=30/(\log(n))^2$ for each $n$ to satisfy the high SNR condition (i.e., $\sigma^2\log(n)\rightarrow 0$). 

Figure~\ref{fig_consistency_to_truth} (Left) illustrates how the MSE value of proposed method varies with $n$, where we use $\lambda$ given by the theoretically optimal value (i.e., $\lambda=1-1/(2\sigma^2\lambda^{\ast}+1)$). The averaged values of 30 repetitions are reported, as well as standard deviations. 
As the benchmark, the black dotted line is the result of $k$-means on the original fully observed dataset, and the blue dashed line is the result of $k$-POD for reference. 
We can see that as $n$ increases, the MSE value of proposed method with the theoretically optimal $\lambda$ tends to zero. For $n$ larger than 3000, the MSE value of proposed method is as low as that of $k$-means, whereas, the MSE value of $k$-POD method is still significantly large. 
This confirms the convergence of estimated cluster centers of proposed method to true cluster centers under conditions given in Theorem~\ref{theorem_converge_to_truth}. 

Moreover, Figure~\ref{fig_consistency_to_truth} (Right) shows the MSE values of proposed method across various choices of $\lambda$, where the red solid line corresponds to the theoretically optimal $\lambda$ and dashed and dotted lines correspond to different scalings of the theoretically optimal $\lambda$. The averaged values of 30 repetitions are reported. 
We can see that as $n$ increases, the theoretically optimal $\lambda$ almost achieves the uniformly lowest MSE value and converges to zero. However, the MSE values associated with non-optimal $\lambda$ remain relatively large even when $n\geq 10^4$. 
This further confirms the optimality of the theoretical choice of $\lambda$ given in Theorem~\ref{theorem_converge_to_truth}.

\begin{figure}
    \centering
    \includegraphics[width=0.33\linewidth]{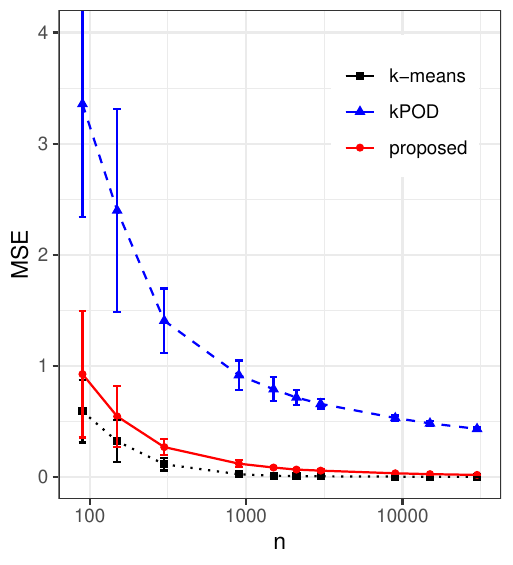}
    \includegraphics[width=0.33\linewidth]{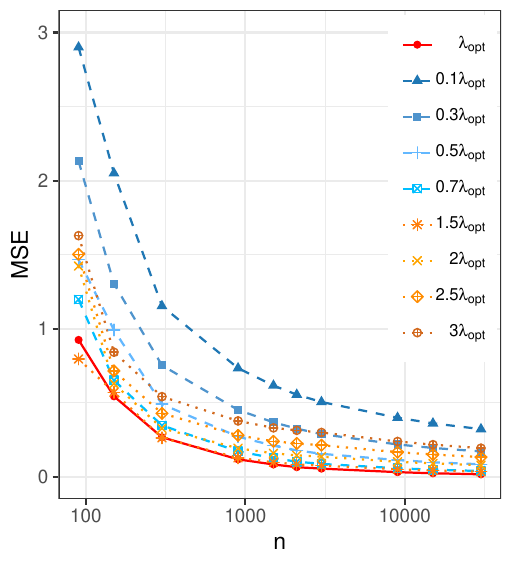}
    \caption{\textbf{Left}: The convergence of estimated cluster centers of proposed method to true cluster centers as sample size $n$ increases. The red solid line is the proposed method with theoretically optimal $\lambda$. The blue dashed line is the result of $k$-POD. The black dotted line is the result of $k$-means on the original fully observed dataset. \textbf{Right}: The MSE values of proposed method with different values of $\lambda$, where $\lambda_{\text{opt}}$ is the theoretically optimal value and $c\lambda_{\text{opt}}$ is $c$ times of theoretically optimal value. }
    \label{fig_consistency_to_truth}
\end{figure}

\section{Applications}
\label{sec_applications}

\subsection{Real-world dataset with artificial missingness}
\label{subsec_application_lymphoma}

In this section, we apply the proposed method to a real-world dataset with artificial missingness. 
The fully observed real dataset is a microarray genomics dataset called \textit{Lymphoma}, which was provided by Jin and Wang \cite{jin2016influential}. It consists of 4026 gene expressions ($p=4026$) and 62 samples ($n=62$), where 42 are Diffuse Large B-Cell Lymphoma (DLBCL), 9 are Follicular Lymphoma (FL), and 11 are Chronic Lymphocytic Leukemia (CLL) ($k=3$). 
Since the original dataset is complete and includes no missingness, we generate artificial missingness to get an incomplete dataset, where three kinds of missing mechanisms used in Section~\ref{subsec_experiment_comparison} are considered, and different hyper-parameters are chosen to meet missing proportions from 10\% to 50\%. 

Table~\ref{table_realdata_lymphoma} illustrates the results of different methods for the \textit{Lymphoma} dataset with different missingness. Here we ignore the multiple imputation method for its extremely long computation time, and exclude FWPD method since it fails to give a reasonable clustering result in almost all settings. 
The tuning parameter $\lambda$ of the proposed method is selected by instability criterion from a set of candidates $\{2,4,6,8,10\}$. 
Since the true cluster centers are unknown, we only report the averaged CER values of 10 repetitions, as well as standard deviations in brackets. The bold font indicates the best result. 
It can be seen that the proposed method performs better than other methods, which has CER values lower than 0.1 in most settings. This result implies the ability of the proposed method to recover the true cluster structure of the large $p$ and small $n$ data with non-random complex missingness. 

\begin{table}[!t]
\centering
\caption{The CER values of different methods for \textit{Lymphoma} datasets (standard deviations in brackets) }
\label{table_realdata_lymphoma}
\begin{adjustbox}{center}
\begin{tabular}{@{}llccccc@{}}
\toprule
 \makecell[l]{Missing\\mechanism} &  \makecell[l]{Missing\\proportion}  & $k$-means & \makecell[c]{Mean\\imputation} & \makecell[c]{KNN\\imputation} & $k$-POD & Proposed \\
\midrule
MNAR0 & 10\% & 0.297 (0.00) & 0.297 (0.00) & 0.297 (0.00) & 0.297 (0.00) & \bf{0.216 (0.13)} \\
  & 30\% & 0.297 (0.00) & 0.279 (0.00) & 0.297 (0.00) & 0.279 (0.00) & \bf{0.029 (0.01)} \\
  & 50\% & 0.297 (0.00) & 0.283 (0.01) & 0.294 (0.01) & 0.295 (0.01) & \bf{0.093 (0.07)} \\
MNAR1 & 10\% & 0.297 (0.00) & 0.295 (0.01) & 0.297 (0.00) & 0.295 (0.01) & \bf{0.080 (0.11)}\\
  & 30\% & 0.297 (0.00) & 0.283 (0.01) & 0.296 (0.00) & 0.283 (0.01) & \bf{0.034 (0.02)} \\
  & 50\% & 0.297 (0.00) & 0.298 (0.01) & 0.296 (0.00) & 0.300 (0.01) & \bf{0.167 (0.11)} \\
MNAR2 & 10\% & 0.297 (0.00) & 0.297 (0.00) & 0.297 (0.00) & 0.297 (0.00) & \bf{0.108 (0.13)} \\
  & 30\% & 0.297 (0.00) & 0.279 (0.00) & 0.291 (0.00) & 0.279 (0.00) & \bf{0.040 (0.01)} \\
  & 50\% & 0.297 (0.00) & 0.279 (0.00) & 0.308 (0.00) & 0.279 (0.00) & \bf{0.067 (0.06)}\\
\bottomrule
\end{tabular}
\end{adjustbox}
\end{table}

\subsection{Real-world incomplete dataset}

In this section, we apply the proposed method to a real-world incomplete dataset. 
We consider an scRNA-seq dataset \textit{Usoskin}, provided by Usoskin et al. \cite{usoskin2015unbiased}. The \textit{Usoskin} dataset contains 622 neuronal cells ($n=622$) that are divided into four sensory subtypes ($k=4$): peptidergic nociceptors (PEP), non-peptidergic nociceptors (NP), neurofilament containing (NF) and tyrosine hydroxylase containing (TH). We here consider 452 genes ($p=452$) of this dataset, where the missingness is less than 90\%. Then, the total missing proportion of the used dataset is about 73\%.

Figure~\ref{fig_realdata_Usoskin} provides the visualization of clustering results of different methods. 
Since the multiple imputation takes an extremely long time and the KNN imputation cannot be applied under a large missing proportion, we instead consider a well-known imputation method specifically for scRNAseq data, \textit{scImpute} proposed by Li et al. \cite{li2018accurate} for comparison. For our method, the best $\lambda\in\{0.001,0.01,0.1,1,10\}$ is selected by the instability criterion. 
For visualization, we use the UMAP method \cite{becht2019dimensionality} to map the high-dimensional data points to a two-dimensional space, before which we impute missingness by zero. 
It can be seen that the proposed method has the lowest CER value, and its clustering result most closely coincides with the true cluster structure. The only clustering error of the proposed method is incorrectly grouping several NF cells to be PEP cells, and incorrectly grouping several NP cells to PEP or TH cells. The total number of mis-clustered cells is only 47. 
In contrast, peer methods have over 100 (or even over 150) cells mis-clustered, where imputation methods almost confuse PEP and NF cells, and $k$-POD fails to discriminate TH cells and PEP cells. 

\begin{figure}[t]
    \centering
    \begin{minipage}{0.3\textwidth}
    \centering
        \includegraphics[width=\linewidth]{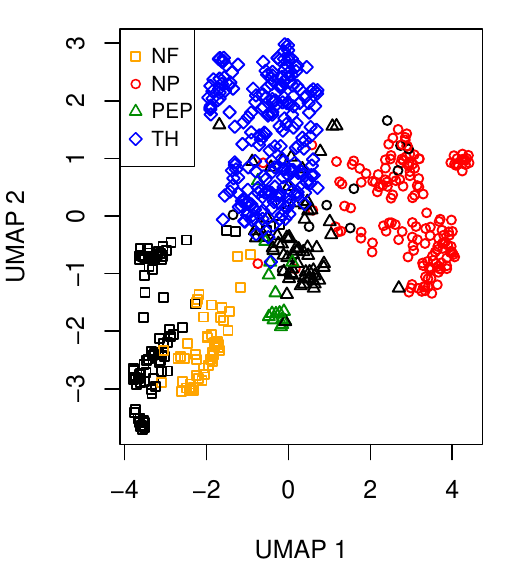}
        {\small \centering{Mean imputation\\(CER = 0.138)} }
    \end{minipage}
    \begin{minipage}{0.3\textwidth}
    \centering
        \includegraphics[width=\linewidth]{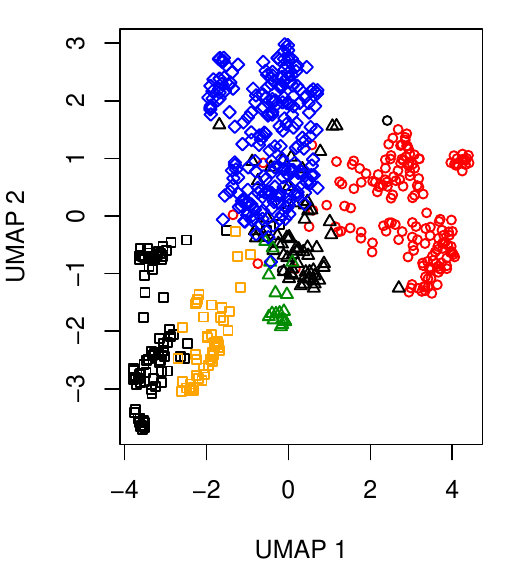}
        {\small scImpute\\(CER = 0.118) }
    \end{minipage}
    \par
    \begin{minipage}{0.3\textwidth}
    \centering
        \includegraphics[width=\linewidth]{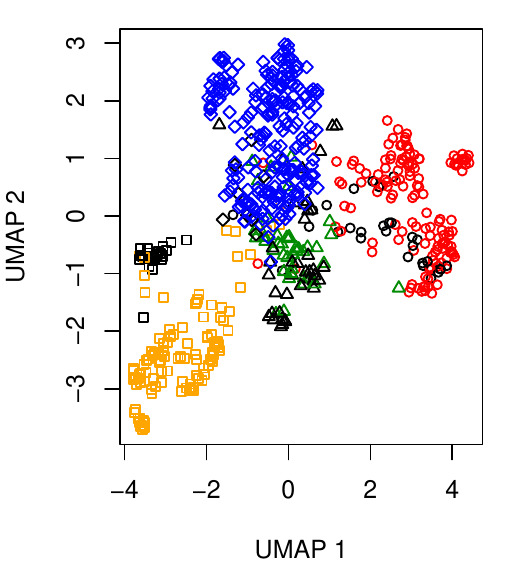}
        {\small $k$-POD\\(CER = 0.132) }
    \end{minipage}
    \begin{minipage}{0.3\textwidth}
    \centering
        \includegraphics[width=\linewidth]{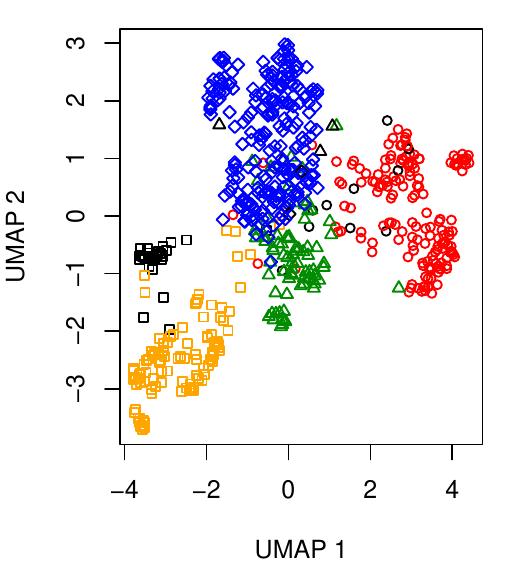}
        {\small \bf{Proposed}\\\bf{(CER = 0.057)}}
    \end{minipage}
    \caption{The UMAP visualization of clustering results of different methods for the \textit{Usoskin} dataset. The shape of points is the true label. The four colors (red, blue, green, orange) represent the correctly estimated labels, while black means mis-clustered points. The reported CER value is the average of 30 repetitions. }
    \label{fig_realdata_Usoskin}
\end{figure}

\section{Conclusion and discussions}
\label{sec_conclusion}

In this paper, we proposed a novel $k$-means clustering for data missing not at random with magnitude-decaying probability.  
Based on the constraint on the imputation values at missing positions, the proposed method alleviates the overestimation of cluster centers caused by simple cluster mean imputation, leading to a better clustering result. 
Moreover, such a penalty can be interpreted as a maximization of missing probability of missing entries, which further reveals the reason of the failure of existing $k$-POD clustering under MNAR mechanisms. 
We also proved the statistical consistency of the estimated cluster centers obtained by the proposed method, showing the ability of our missing data clustering method to recover the underlying cluster structure of fully observed data. 
The optimization was solved by an alternative minimization algorithm, which has a convergence guarantee and linear computation time. 
Simulation experiments verified the effect of the proposed method in improving clustering results and reducing the bias of estimated cluster centers, and confirmed the statistical consistency of estimated cluster centers to true cluster centers. 
Applications to real-world missing data further illustrate the utility of the proposed method.

There are still some limitations for further improvement. 
First, the proposed method is tailored for a specific ``magnitude-decaying" scenario of MNAR mechanisms, where data is more likely to be missing at positions with smaller absolute values, and we model it by Assumption~\ref{assumption_def_MNAR0}. Although our experimental results show that the proposed method remains robust even when Assumption~\ref{assumption_def_MNAR0} does not hold, it is worthwhile to design methods specifically for other MNAR mechanisms in the future. 
Second, the statistical consistency established in this work relies on Assumption~\ref{assumption_def_MNAR0} and Gaussian mixture distribution with true centers differing in each dimension. An important direction for further work is to investigate how to relax these conditions and to prove statistical properties under broader MNAR assumptions and other data distributions. 
Third, since the proposed method is a $k$-means-type clustering method, it still retains limitations common to $k$-means, such as for data with overlapping or imbalanced clusters. It will be interesting to extend our idea to non-linear variants of $k$-means, so as to handle more complex cluster structures.

\bibliographystyle{abbrvnat}
\bibliography{main}

\newpage
\appendix

\section{Proof of Theorem~\ref{theorem_consistency}}

\begin{proof}
    First, we prove the excess risk bound. 
    For any $n\in \mathbb{N}_+$, according to the definition of $\bm{M}^{\ast,\lambda}$,
    we have the excess risk is bounded by
    \begin{align*}
        & L^{\lambda}(\widehat{\bm{\mathrm{M}}}^{\lambda}) - \min_{\bm{M}\in\mathcal{M}} L^{\lambda}(\bm{M}) \\
        &=
        L^{\lambda}(\widehat{\bm{\mathrm{M}}}^{\lambda}) - L^{\lambda}(\bm{M}^{\ast,\lambda})\\
        &=L^{\lambda}(\widehat{\bm{\mathrm{M}}}^{\lambda}) - \widehat{L}_n^{\lambda}(\widehat{\bm{\mathrm{M}}}^{\lambda}) + \widehat{L}_n^{\lambda}(\widehat{\bm{\mathrm{M}}}^{\lambda}) - \widehat{L}_n^{\lambda}(\bm{M}^{\ast,\lambda}) + \widehat{L}_n^{\lambda}(\bm{M}^{\ast,\lambda}) - L^{\lambda}(\bm{M}^{\ast,\lambda})\\
        &\leq 2\cdot \sup_{\bm{M}\in \mathcal{M}} \left| L^{\lambda}(\bm{M}) -  \widehat{L}_n^{\lambda}(\bm{M}) \right| + \widehat{L}_n^{\lambda}(\widehat{\bm{\mathrm{M}}}^{\lambda}) - \widehat{L}_n^{\lambda}(\bm{M}^{\ast,\lambda}) \\
        &\leq 2\cdot \sup_{\bm{M}\in \mathcal{M}} \left| L^{\lambda}(\bm{M}) -  \widehat{L}_n^{\lambda}(\bm{M}) \right|,
    \end{align*}
    where the final inequality is because $\widehat{L}_n^{\lambda}(\widehat{\bm{\mathrm{M}}}^{\lambda}) - \widehat{L}_n^{\lambda}(\bm{M}^{\ast,\lambda}) \leq 0$ according to the definition of $\widehat{\bm{\mathrm{M}}}^{\lambda}$. 
    
    Then, by using Lemma~\ref{lemma_uniform_convergence} and the definition of $\bm{M}^{\ast,\lambda}$, we obtain for any $\tilde{\epsilon}>0$,
    \begin{align*}
        \lim_{n\rightarrow\infty} \text{Pr}\left(  L^{\lambda}(\widehat{\bm{\mathrm{M}}}^{\lambda}) - L^{\lambda}(\bm{M}^{\ast,\lambda}) >  \tilde{\epsilon} \right) = 0.
    \end{align*}
    According to Lemma~\ref{lemma_identifiability}, the identifiability of $\bm{M}^{\ast,\lambda}$ implies that for any $\epsilon>0$, there exists $\tilde{\epsilon}>0$ such that
    \begin{align*}
        \textnormal{Pr}\left( \|\text{vec}(\widehat{\bm{\mathrm{M}}}^{\lambda}) - \text{vec}(\bm{M}^{\ast,\lambda}) \|_2 >\epsilon \right) 
        \leq \textnormal{Pr}\left( L^{\lambda}(\widehat{\bm{\mathrm{M}}}^{\lambda}) - L^{\lambda}(\bm{M}^{\ast,\lambda}) >  \tilde{\epsilon} \right). 
    \end{align*}
    It follows that $\lim_{n\rightarrow\infty} \textnormal{Pr}\left( \|\text{vec}(\widehat{\bm{\mathrm{M}}}^{\lambda}) - \text{vec}(\bm{M}^{\ast,\lambda}) \|_2  >\epsilon \right) = 0$, which completes the proof. 
\end{proof}

\section{Proof of Theorem~\ref{theorem_converge_to_truth}}
\begin{proof}
Our aim is to prove $\widehat{\bm{\mathrm{M}}}^{\lambda}$ converges to $\bm{M}^{\ast\ast}$ in probability. 
Recall that $\bm{M}^{\ast\ast}$ refers to true cluster centers. We denote the cluster label of any $\bm{\mathrm{x}}_i$ associated with $\bm{M}^{\ast\ast}$ by 
\begin{align*}
    \mathrm{z}_i^{\ast\ast}=\mathop{\arg\min}_{l=1,\dots,k}  \| \bm{\mathrm{x}}_{i} - \bm{\mu}_{l}^{\ast\ast} \|_2^2,  
\end{align*}
and write $\bm{\mathrm{U}}^{\ast\ast}=(\mathrm{u}_{il}^{\ast\ast})_{n\times p}$, where $\mathrm{u}_{il}^{\ast\ast}=\mathds{1}(\mathrm{z}_i^{\ast\ast}=l)$. 
Moreover, we here use the formula of empirical loss function $\widehat{L}_n^{\lambda}$ that involves cluster assignment $\bm{U}$, i.e.,  
\begin{align*}
    \widehat{L}_n^{\lambda}(\bm{U},\bm{M})= \frac{1}{n}\sum_{i=1}^{n} \sum_{l=1}^{k} u_{il} \sum_{j=1}^{p} \mathrm{r}_{ij}( \mathrm{x}_{ij} - \mu_{lj} )^2  + \lambda (1-\mathrm{r}_{ij})\lambda \mu_{lj}^2,
\end{align*}
and denoted by $\widehat{\bm{\mathrm{U}}}^{\lambda}=(\hat{\mathrm{u}}_{il}^{\lambda})_{n\times p}$ the estimated cluster assignment associated with $\widehat{\bm{\mathrm{M}}}^{\lambda}$, i.e., 
\begin{align*}
    \hat{\mathrm{u}}_{il}^{\lambda}=\mathds{1}\left( l=\mathop{\arg\min}_{t=1,\dots,k} \sum_{j=1}^{p} \mathrm{r}_{ij}( \mathrm{x}_{ij} - \hat{\mu}_{tj}^{\lambda} )^2  + \lambda (1-\mathrm{r}_{ij})\lambda (\hat{\mu}_{tj}^{\lambda})^2 \right).
\end{align*}
Then, it suffices to prove $(\widehat{\bm{\mathrm{U}}}^{\lambda},\widehat{\bm{\mathrm{M}}}^{\lambda})$ converges to $(\bm{\mathrm{U}}^{\ast\ast},\bm{M}^{\ast\ast})$ in probability. 
Since the estimated $\bm{M}$ and $\bm{U}$ are one-to-one corresponded, then it suffices to prove: 
\begin{itemize}
    \item[(1)] Given $\widehat{\bm{\mathrm{M}}}^{\lambda}=\bm{M}^{\ast\ast}$, the $\bm{\mathrm{U}}^{\ast\ast}$ minimizes $\widehat{L}_n^{\lambda}(\bm{U},\bm{M}^{\ast\ast})$,
    \item[(2)] Given $\widehat{\bm{\mathrm{U}}}^{\lambda}=\bm{\mathrm{U}}^{\ast\ast}$, the $\widehat{\bm{\mathrm{M}}}^{\lambda}$ converges to $\bm{M}^{\ast\ast}$; 
\end{itemize}
which are given by Lemma~\ref{lemma_proof_converge_to_truth_about_U} and Lemma~\ref{lemma_proof_converge_to_truth_about_M}, respectively. 

\end{proof}

\section{Other lemmas and proofs}

\begin{lemma}
\label{lemma_uniform_convergence}
    Under the above settings, we have the uniform convergence of $\widehat{L}_n^{\lambda}$, i.e., for any $\tilde{\epsilon}>0$, 
    \begin{align*}
        \lim_{n\rightarrow\infty} \textnormal{Pr}\left( \sup_{\bm{M}\in \mathcal{M}} \left| L^{\lambda}(\bm{M}) - \widehat{L}^{\lambda}_n(\bm{M}) \right| >  \tilde{\epsilon} \right) =0. 
    \end{align*}
\end{lemma}
\begin{proof}
    Consider a function class on $\mathcal{X}\times \{0,1\}^p$ as follows: 
    \begin{align*}
        \mathcal{G}_{k}&=\bigg\{  g_{k,\bm{M}}:(\bm{x},\bm{r})\mapsto \min_{l=1,\dots,k} \sum_{j=1}^{p} r_j(x_j - \mu_{lj})^2 + \lambda(1-r_j)\mu_{lj}^2 \;\bigg|\;  \forall \bm{M}\in\mathcal{M}, \|\bm{\mu}_l\|_2\leq B, l=1,\dots,k \bigg\}.
    \end{align*}
    For a size-$n$ random sample of missing data $(\bm{\mathrm{X}},\bm{\mathrm{R}})=\{(\bm{\mathrm{x}}_1,\bm{\mathrm{r}}_1),\dots,(\bm{\mathrm{x}}_n,\bm{\mathrm{r}}_n)\}$, we can write  
    \begin{align*}
        L^{\lambda}(\bm{M})=\mathbb{E}_{\bm{\mathrm{x}}_1,\bm{\mathrm{r}}_1}\left[ g_{k,\bm{M}}(\bm{\mathrm{x}}_1,\bm{\mathrm{r}}_1) \right]
        \quad \textnormal{and} \quad 
        \widehat{L}^{\lambda}_n(\bm{M})=\frac{1}{n}\sum_{i=1}^{n} g_{k,\bm{M}}(\bm{\mathrm{x}}_i,\bm{\mathrm{r}}_i).
    \end{align*}

    First, we prove the expectation $\mathbb{E}_{\bm{\mathrm{X}},\bm{\mathrm{R}}} \left[ \sup_{\bm{M}\in \mathcal{M}} \left| L^{\lambda}(\bm{M}) - \widehat{L}^{\lambda}_n(\bm{M}) \right| \right]$ is $O(n^{-1/2})$. 

    To this end, we define a function class as follows: 
    \begin{align*}
        \mathcal{G}&=\bigg\{ g_{\bm{\mu}}:(\bm{x},\bm{r})\mapsto \sum_{j=1}^{p} r_j(x_j - \mu_{j})^2 + \lambda(1-r_j)\mu_{j}^2 \;\bigg|\; \forall \bm{\mu} \in \mathcal{X}, \|\bm{\mu}\|_2\leq B \bigg\}.
    \end{align*}
    Let $\mathfrak{R}_n(\mathcal{G}_k)$ and $\mathfrak{R}_n(\mathcal{G})$ be the Rademacher complexity of $\mathcal{G}_k$ and $\mathcal{G}$, respectively, then according to Theorem 12 of \cite{bartlett2002rademacher}, we have 
    \begin{align*}
        \mathfrak{R}_n(\mathcal{G}_k)\leq k\cdot \mathfrak{R}_n(\mathcal{G}). 
    \end{align*}
    Moreover, we define an envelop function $G:\mathcal{X}\times \{0,1\}^p\mapsto \mathbb{R}$ by 
    \begin{align*}
        G(\bm{x},\bm{r})=\sum_{j=1}^{p} x_j^2 + \sum_{j=1}^{p} 2B|x_j| + (\lambda +1)B^2.
    \end{align*}
    Then for any $g_{\bm{\mu}}\in\mathcal{G}$, we have $|g_{\bm{\mu}}(\bm{x},\bm{r})|\leq G(\bm{x},\bm{r})$ holds for any $(\bm{x},\bm{r})\in \mathcal{X}\times \{0,1\}^p$. 
    Under Assumption~\ref{assumption_X_compact_sub-Gaussian} (i.e., $\bm{\mathrm{x}}_1$ is sub-Gaussian), we have $\mathbb{E}_{\bm{\mathrm{x}}_1}[\|\bm{\mathrm{x}}_1\|_2^4]<\infty$, it follows that $\mathbb{E}_{\bm{\mathrm{x}}_1,\bm{\mathrm{r}}_1}[(G(\bm{\mathrm{x}}_1,\bm{\mathrm{r}}_1))^2]<\infty$. We write 
    \begin{align*}
        \sigma_G^2=\mathbb{E}_{\bm{\mathrm{x}}_1,\bm{\mathrm{r}}_1}[(G(\bm{\mathrm{x}}_1,\bm{\mathrm{r}}_1))^2]. 
    \end{align*} 
    Then, according to Equations~(3.8)-(3.13) of \cite{mohri2018foundations}, we have 
    \begin{align*}
        \mathbb{E}_{\bm{\mathrm{X}},\bm{\mathrm{R}}}\left[ \sup_{\bm{M}\in \mathcal{M}} \left| L^{\lambda}(\bm{M}) - \widehat{L}^{\lambda}_n(\bm{M}) \right| \right]
        \leq 2\mathfrak{R}_n(\mathcal{G}_k)
        \leq 2k\mathfrak{R}_n(\mathcal{G}),
    \end{align*}
    Thus, it suffices to bound $\mathfrak{R}_n(\mathcal{G})$. 
    
    According to the Dudley’s entropy integral bound (Theorem~5.22 of \cite{wainwright2019high}), we have
    \begin{align*}
        \mathfrak{R}_n(\mathcal{G}) \leq \mathbb{E}_{\bm{\mathrm{X}},\bm{\mathrm{R}}} \left[  \frac{24}{\sqrt{n}} \int_{0}^{\infty} \sqrt{\log( N( \tau, \mathcal{G}, \mathcal{L}_2(\mathbb{P}_n) )) }\; d\tau \right],
    \end{align*}
    where $N( \tau, \mathcal{G}, \mathcal{L}_2(\mathbb{P}_n) )$ is the $\tau$-covering number of $\mathcal{G}$ with respect to the $\mathcal{L}_2(\mathbb{P}_n)$ norm. 
    Notice that since $G$ is the envelop function of $\mathcal{G}$, then $N( \tau, \mathcal{G}, \mathcal{L}_2(\mathbb{P}_n) )=1$ 
    for $\tau\geq \|G\|_{\mathcal{L}_2(\mathbb{P}_n)}$, implying the integral function equals to zero. It follows that 
    \begin{align*}
        &\int_{0}^{\infty} \sqrt{\log( N( \tau, \mathcal{G}, \mathcal{L}_2(\mathbb{P}_n) )) }\; d\tau 
        =\int_{0}^{\|G\|_{\mathcal{L}_2(\mathbb{P}_n)}} \sqrt{\log( N( \tau, \mathcal{G}, \mathcal{L}_2(\mathbb{P}_n) )) }\; d\tau.
    \end{align*}
    Moreover, Lemma~\ref{lemma_VCdim} shows that $\mathcal{G}$ has its VC-dimension $\text{VCdim}(\mathcal{G})<\infty$. Then, according to the covering number Theorem (Theorem 2.6.7 of \cite{van1996weak}), we have 
    \begin{align*}
        N( \delta\|G\|_{\mathcal{L}_2(\mathbb{P}_n)} , \mathcal{G}, \mathcal{L}_2(\mathbb{P}_n) )
        &\leq C_0 \cdot \text{VCdim}(\mathcal{G})\cdot (16e)^{\text{VCdim}(\mathcal{G})} \left(\frac{1}{\delta}\right)^{2\text{VCdim}(\mathcal{G})}
        \leq \left(\frac{C_1}{\delta}\right)^{2\text{VCdim}(\mathcal{G})}
    \end{align*}
    where $C_0$ is a constant and $C_1$ is a constant satisfying $ C_0 \cdot \text{VCdim}(\mathcal{G})\cdot (16e)^{\text{VCdim}(\mathcal{G})}\leq C_1^{2\text{VCdim}(\mathcal{G})}$. 
    Thus, we obtain 
    \begin{align*}
        \mathfrak{R}_n(\mathcal{G}) 
        &\leq \mathbb{E}_{\bm{\mathrm{X}},\bm{\mathrm{R}}} \left[  \frac{24}{\sqrt{n}} \int_{0}^{\|G\|_{\mathcal{L}_2(\mathbb{P}_n)}}  \sqrt{\log( N( \tau, \mathcal{G}, \mathcal{L}_2(\mathbb{P}_n) )) }\; d\tau \right] \\
        &= \mathbb{E}_{\bm{\mathrm{X}},\bm{\mathrm{R}}} \left[  \frac{24}{\sqrt{n}} \|G\|_{\mathcal{L}_2(\mathbb{P}_n)}\cdot \int_{0}^{1} \sqrt{ \log\left(\left(C_1/\delta \right)^{2\text{VCdim}(\mathcal{G})} \right)  }  \; d\delta
        \right]\\
        &= \mathbb{E}_{\bm{\mathrm{X}},\bm{\mathrm{R}}} \left[  \frac{24}{\sqrt{n}} \|G\|_{\mathcal{L}_2(\mathbb{P}_n)}\cdot 
        \sqrt{ 2\text{VCdim}(\mathcal{G})} \cdot \int_{0}^{1} \sqrt{\log(C_1/\delta)} \; d\delta
        \right]\\ 
        &:= \mathbb{E}_{\bm{\mathrm{X}},\bm{\mathrm{R}}} \left[  \frac{24}{\sqrt{n}} \|G\|_{\mathcal{L}_2(\mathbb{P}_n)} \cdot \sqrt{ 2\text{VCdim}(\mathcal{G})} \cdot C_2 \right],
    \end{align*}
    where $C_2=\int_{0}^{1} \sqrt{\log(C_1/\delta)} \; d\delta <\infty$ is a constant. 
    Because 
    \begin{align*}
        \mathbb{E}_{\bm{\mathrm{X}},\bm{\mathrm{R}}} \left[ \|G\|_{\mathcal{L}_2(\mathbb{P}_n)} \right] 
        &= \mathbb{E}_{\bm{\mathrm{X}},\bm{\mathrm{R}}} \left[ \left( \frac{1}{n} \sum_{i=1}^{n} (G(\bm{\mathrm{x}}_i,\bm{\mathrm{r}}_i) )^2 \right)^{1/2} \right] \\
        &\leq \left( \mathbb{E}_{\bm{\mathrm{X}},\bm{\mathrm{R}}} \left[ \frac{1}{n} \sum_{i=1}^{n} (G(\bm{\mathrm{x}}_i,\bm{\mathrm{r}}_i) )^2 \right] \right)^{1/2} \\
        &= \left( \mathbb{E}_{\bm{\mathrm{x}}_1,\bm{\mathrm{r}}_1} \left[ (G(\bm{\mathrm{x}}_1,\bm{\mathrm{r}}_1) )^2 \right] \right)^{1/2} \\
        &= \sigma_G,
    \end{align*}
    then $\mathfrak{R}_n(\mathcal{G})$ is bounded by
    \begin{align*}
        &\mathfrak{R}_n(\mathcal{G})
        \leq \frac{24\sigma_G}{\sqrt{n}}\cdot \sqrt{ 2\text{VCdim}(\mathcal{G})} \cdot C_2, 
    \end{align*}
    which follows that 
    \begin{align}
        &\mathbb{E}_{\bm{\mathrm{X}},\bm{\mathrm{R}}}\left[ \sup_{\bm{M}\in \mathcal{M}} \left| L^{\lambda}(\bm{M}) - \widehat{L}^{\lambda}_n(\bm{M}) \right| \right] 
        \leq \frac{48 k \cdot \sigma_G}{\sqrt{n}}\cdot \sqrt{ 2\text{VCdim}(\mathcal{G})} \cdot C_2.
        \label{proof_uniform_convergence_eq_1}
    \end{align}

    Second, we prove the tail probability $\text{Pr}\left( \sup_{M\in \mathcal{M}} \left| L^{\lambda}(\bm{M}) - \widehat{L}_n^{\lambda}(\bm{M}) \right| > \tilde{\epsilon} \right)$ tends to zero as $n\rightarrow\infty$ for any $\tilde{\epsilon}>0$. 
 
    To this end, for any fixed $n\in\mathbb{N}_+$, we define 
    \begin{align*}
        &g_{k,\bm{M}}^{-}(\bm{x},\bm{r})=g_{k,\bm{M}}(\bm{x},\bm{r})\cdot \mathds{1}(G(\bm{x},\bm{r})\leq \sqrt{n}) \\
        &g_{k,\bm{M}}^{+}(\bm{x},\bm{r})=g_{k,\bm{M}}(\bm{x},\bm{r})\cdot \mathds{1}(G(\bm{x},\bm{r})> \sqrt{n}).
    \end{align*}
    Because 
    \begin{align*}
        &\sup_{M\in \mathcal{M}} \left| L^{\lambda}(\bm{M}) - \widehat{L}_n^{\lambda}(\bm{M}) \right| \\
        &= \sup_{M\in \mathcal{M}} \left| \mathbb{E}_{\bm{\mathrm{x}}_1,\bm{\mathrm{r}}_1} \left[ g_{k,\bm{M}}(\bm{\mathrm{x}}_1,\bm{\mathrm{r}}_1) \right] - \frac{1}{n}\sum_{i=1}^{n} g_{k,\bm{M}}(\bm{\mathrm{x}}_i,\bm{\mathrm{r}}_i)  \right| \\
        &\leq \sup_{M\in \mathcal{M}} \left| \mathbb{E}_{\bm{\mathrm{x}}_1,\bm{\mathrm{r}}_1} \left[ g_{k,\bm{M}}^{-}(\bm{\mathrm{x}}_1,\bm{\mathrm{r}}_1) \right] - \frac{1}{n}\sum_{i=1}^{n} g_{k,\bm{M}}^{-}(\bm{\mathrm{x}}_i,\bm{\mathrm{r}}_i)  \right|  + \sup_{M\in \mathcal{M}} \left| \mathbb{E}_{\bm{\mathrm{x}}_1,\bm{\mathrm{r}}_1} \left[ g_{k,\bm{M}}^{+}(\bm{\mathrm{x}}_1,\bm{\mathrm{r}}_1) \right] - \frac{1}{n}\sum_{i=1}^{n} g_{k,\bm{M}}^{+}(\bm{\mathrm{x}}_i,\bm{\mathrm{r}}_i)  \right| ,
    \end{align*}
    then for any $\tilde{\epsilon}>0$, 
    \begin{align*}
        &\text{Pr}\left( \sup_{M\in \mathcal{M}} \left| L^{\lambda}(\bm{M}) - \widehat{L}_n^{\lambda}(\bm{M}) \right| > \tilde{\epsilon} \right) \\
        &\leq \text{Pr}\left( \sup_{M\in \mathcal{M}} \left| \mathbb{E}_{\bm{\mathrm{x}}_1,\bm{\mathrm{r}}_1} \left[ g_{k,\bm{M}}^{-}(\bm{\mathrm{x}}_1,\bm{\mathrm{r}}_1) \right] - \frac{1}{n}\sum_{i=1}^{n} g_{k,\bm{M}}^{-}(\bm{\mathrm{x}}_i,\bm{\mathrm{r}}_i)  \right| > \frac{\tilde{\epsilon}}{2} \right) \\
        &\quad + \text{Pr}\left( \sup_{M\in \mathcal{M}} \left| \mathbb{E}_{\bm{\mathrm{x}}_1,\bm{\mathrm{r}}_1} \left[ g_{k,\bm{M}}^{+}(\bm{\mathrm{x}}_1,\bm{\mathrm{r}}_1) \right] - \frac{1}{n}\sum_{i=1}^{n} g_{k,\bm{M}}^{+}(\bm{\mathrm{x}}_i,\bm{\mathrm{r}}_i)  \right| > \frac{\tilde{\epsilon}}{2} \right)\\
        &:= \text{(I)} + \text{(II)}.
    \end{align*}
    Thus, it suffices to prove the two terms converge to zero as $n\rightarrow\infty$. 
    
    For $\text{(I)}$, because 
    \begin{align*}
        \| g_{k,\bm{M}}^{-} \|_{\infty} 
        = \sup_{\bm{x},\bm{r}}| g_{k,\bm{M}}^{-}(\bm{x},\bm{r})|
        \leq \sup_{\bm{x},\bm{r}}\left| G(\bm{x},\bm{r})\cdot \mathds{1}(G(\bm{x},\bm{r})\leq \sqrt{n}) \right|
        \leq \sqrt{n}
    \end{align*}
    and 
    \begin{align*}
        &\mathbb{E}_{\bm{\mathrm{X}},\bm{\mathrm{R}}}\left[ \sup_{\bm{M}\in\mathcal{M}} \frac{1}{n}\sum_{i=1}^{n} \left( g_{k,\bm{M}}^{-}(\bm{\mathrm{x}}_i,\bm{\mathrm{r}}_i)  \right)^2 \right] \\
        &\leq \mathbb{E}_{\bm{\mathrm{X}},\bm{\mathrm{R}}}\left[ \frac{1}{n}\sum_{i=1}^{n} \left( G(\bm{\mathrm{x}}_i,\bm{\mathrm{r}}_i)\right)^2 \cdot \mathds{1}(G(\bm{\mathrm{x}}_i,\bm{\mathrm{r}}_i)\leq \sqrt{n}) \right] \\
        &\leq \mathbb{E}_{\bm{\mathrm{x}}_1,\bm{\mathrm{r}}_1}\left[ \left( G(\bm{\mathrm{x}}_1,\bm{\mathrm{r}}_1)\right)^2 \right] 
        = \sigma_G^2,
    \end{align*}
    then according to Talagrand inequality (Theorem 3.27 of \cite{wainwright2019high}), we have 
    \begin{align}
        &\text{Pr}\left(  \sup_{\bm{M}\in\mathcal{M}} \frac{1}{n}\sum_{i=1}^{n} g_{k,\bm{M}}^{-}(\bm{\mathrm{x}}_i,\bm{\mathrm{r}}_i) 
        - \mathbb{E}_{\bm{\mathrm{X}},\bm{\mathrm{R}}}\left[ \sup_{\bm{M}\in\mathcal{M}} \frac{1}{n}\sum_{i=1}^{n} g_{k,\bm{M}}^{-}(\bm{\mathrm{x}}_i,\bm{\mathrm{r}}_i) \right] \geq \frac{\tilde{\epsilon}}{4}  \right)\notag \\
        &\leq 2\exp\left(  \frac{ -n(\tilde{\epsilon}/4)^2 }{ 
            8e\cdot \sigma_G^2 + 4\sqrt{n}(\tilde{\epsilon}/4) 
         } \right)\notag \\
         &= 2 \exp\left( \frac{ -n\tilde{\epsilon}^2 }{ 128e\cdot \sigma_G^2 + 16\sqrt{n}\tilde{\epsilon} } \right).
         \label{proof_uniform_convergence_eq_2}
    \end{align}
    Moreover, because 
    \begin{align*}
        &\sup_{\bm{M}\in\mathcal{M}} \left\{ \mathbb{E}_{\bm{\mathrm{X}},\bm{\mathrm{R}}}\left[ \sup_{\bm{M}'\in\mathcal{M}} \frac{1}{n}\sum_{i=1}^{n} g_{k,\bm{M}'}^{-}(\bm{\mathrm{x}}_i,\bm{\mathrm{r}}_i) \right] - \mathbb{E}_{\bm{\mathrm{x}}_1,\bm{\mathrm{r}}_1} \left[ g_{k,\bm{M}}^{-}(\bm{\mathrm{x}}_1,\bm{\mathrm{r}}_1) \right] \right\} \\
        &=\mathbb{E}_{\bm{\mathrm{X}},\bm{\mathrm{R}}}\left[ \sup_{\bm{M}'\in\mathcal{M}} \frac{1}{n}\sum_{i=1}^{n} g_{k,\bm{M}'}^{-}(\bm{\mathrm{x}}_i,\bm{\mathrm{r}}_i) \right] - \inf_{\bm{M}\in\mathcal{M}} \mathbb{E}_{\bm{\mathrm{x}}_1,\bm{\mathrm{r}}_1} \left[ g_{k,\bm{M}}^{-}(\bm{\mathrm{x}}_1,\bm{\mathrm{r}}_1) \right] \\
        &= \mathbb{E}_{\bm{\mathrm{X}},\bm{\mathrm{R}}}\left[ \sup_{\bm{M}\in\mathcal{M}} \left\{ \frac{1}{n}\sum_{i=1}^{n} g_{k,\bm{M}}^{-}(\bm{\mathrm{x}}_i,\bm{\mathrm{r}}_i) - \mathbb{E}_{\bm{\mathrm{x}}_1,\bm{\mathrm{r}}_1} \left[ g_{k,\bm{M}}^{-}(\bm{\mathrm{x}}_1,\bm{\mathrm{r}}_1) \right]  \right\} \right] \\
        &\leq \mathbb{E}_{\bm{\mathrm{X}},\bm{\mathrm{R}}}\left[ \sup_{\bm{M}\in\mathcal{M}} \left| L^{\lambda}(\bm{M}) - \widehat{L}_n^{\lambda}(\bm{M}) \right| \right] \\
        &\leq \frac{48 k \cdot \sigma_G}{\sqrt{n}}\cdot \sqrt{ 2\text{VCdim}(\mathcal{G})} \cdot C_2,
    \end{align*}
    where the last inequality is due to Eq.(\ref{proof_uniform_convergence_eq_1}), then there exists $N_1\in\mathbb{N}_+$ such that for any $n\geq N_1$, 
    \begin{align}
        &\sup_{\bm{M}\in\mathcal{M}} \left\{ \mathbb{E}_{\bm{\mathrm{X}},\bm{\mathrm{R}}}\left[ \sup_{\bm{M}'\in\mathcal{M}} \frac{1}{n}\sum_{i=1}^{n} g_{k,\bm{M}'}^{-}(\bm{\mathrm{x}}_i,\bm{\mathrm{r}}_i) \right] - \mathbb{E}_{\bm{\mathrm{x}}_1,\bm{\mathrm{r}}_1} \left[ g_{k,\bm{M}}^{-}(\bm{\mathrm{x}}_1,\bm{\mathrm{r}}_1) \right] \right\} 
        < \frac{\tilde{\epsilon}}{4}. 
     \label{proof_uniform_convergence_eq_3}
    \end{align}
    Combining Eq.(\ref{proof_uniform_convergence_eq_2})-(\ref{proof_uniform_convergence_eq_3}) leads to 
    \begin{align*}
        \text{(I)}
        &\leq 2\cdot \text{Pr}\left( \sup_{M\in \mathcal{M}} \mathbb{E}_{\bm{\mathrm{x}}_1,\bm{\mathrm{r}}_1} \left[ g_{k,\bm{M}}^{-}(\bm{\mathrm{x}}_1,\bm{\mathrm{r}}_1) \right] - \frac{1}{n}\sum_{i=1}^{n} g_{k,\bm{M}}^{-}(\bm{\mathrm{x}}_i,\bm{\mathrm{r}}_i)  > \frac{\tilde{\epsilon}}{2} \right) \\
        &\leq 2\cdot \text{Pr}\left(  
        \sup_{\bm{M}\in\mathcal{M}} \frac{1}{n}\sum_{i=1}^{n} g_{k,\bm{M}}^{-}(\bm{\mathrm{x}}_i,\bm{\mathrm{r}}_i) 
        - \mathbb{E}_{\bm{\mathrm{X}},\bm{\mathrm{R}}}\left[ \sup_{\bm{M}\in\mathcal{M}} \frac{1}{n}\sum_{i=1}^{n} g_{k,\bm{M}}^{-}(\bm{\mathrm{x}}_i,\bm{\mathrm{r}}_i) \right]  \right. \\
        & \quad\quad\quad\quad + \left. \sup_{\bm{M}\in\mathcal{M}} \left\{ \mathbb{E}_{\bm{\mathrm{X}},\bm{\mathrm{R}}}\left[ \sup_{\bm{M}'\in\mathcal{M}} \frac{1}{n}\sum_{i=1}^{n} g_{k,\bm{M}'}^{-}(\bm{\mathrm{x}}_i,\bm{\mathrm{r}}_i) \right] - \mathbb{E}_{\bm{\mathrm{x}}_1,\bm{\mathrm{r}}_1} \left[ g_{k,\bm{M}}^{-}(\bm{\mathrm{x}}_1,\bm{\mathrm{r}}_1) \right] \right\} \geq  \frac{\tilde{\epsilon}}{2} \right)\\
        &\leq 2\cdot \text{Pr}\left(  
        \sup_{\bm{M}\in\mathcal{M}} \frac{1}{n}\sum_{i=1}^{n} g_{k,\bm{M}}^{-}(\bm{\mathrm{x}}_i,\bm{\mathrm{r}}_i) 
        - \mathbb{E}_{\bm{\mathrm{X}},\bm{\mathrm{R}}}\left[ \sup_{\bm{M}\in\mathcal{M}} \frac{1}{n}\sum_{i=1}^{n} g_{k,\bm{M}}^{-}(\bm{\mathrm{x}}_i,\bm{\mathrm{r}}_i) \right]  + \frac{\tilde{\epsilon}}{4}  
        \geq  \frac{\tilde{\epsilon}}{2} \right) \\
        &= 2\cdot \text{Pr}\left(  
        \sup_{\bm{M}\in\mathcal{M}} \frac{1}{n}\sum_{i=1}^{n} g_{k,\bm{M}}^{-}(\bm{\mathrm{x}}_i,\bm{\mathrm{r}}_i) 
        - \mathbb{E}_{\bm{\mathrm{X}},\bm{\mathrm{R}}}\left[ \sup_{\bm{M}\in\mathcal{M}} \frac{1}{n}\sum_{i=1}^{n} g_{k,\bm{M}}^{-}(\bm{\mathrm{x}}_i,\bm{\mathrm{r}}_i) \right]  \geq \frac{\tilde{\epsilon}}{4} \right) \\
        &\leq 2 \exp\left( \frac{ -n\tilde{\epsilon}^2 }{ 128e\cdot \sigma_G^2 + 16\sqrt{n}\tilde{\epsilon} } \right),
    \end{align*}
    which implies that $\text{(I)}\rightarrow 0$ as $n\rightarrow\infty$.  

    For $\text{(II)}$, because 
    \begin{align*}
        &\sup_{M\in \mathcal{M}} \left| \mathbb{E}_{\bm{\mathrm{x}}_1,\bm{\mathrm{r}}_1} \left[ g_{k,\bm{M}}^{+}(\bm{\mathrm{x}}_1,\bm{\mathrm{r}}_1) \right] - \frac{1}{n}\sum_{i=1}^{n} g_{k,\bm{M}}^{+}(\bm{\mathrm{x}}_i,\bm{\mathrm{r}}_i)  \right| \\
        &\leq \sup_{M\in \mathcal{M}} \left| \mathbb{E}_{\bm{\mathrm{x}}_1,\bm{\mathrm{r}}_1} \left[ g_{k,\bm{M}}^{+}(\bm{\mathrm{x}}_1,\bm{\mathrm{r}}_1) \right] \right| 
        + \sup_{M\in \mathcal{M}} \left| \frac{1}{n}\sum_{i=1}^{n} g_{k,\bm{M}}^{+}(\bm{\mathrm{x}}_i,\bm{\mathrm{r}}_i)  \right| \\
        &\leq \sup_{M\in \mathcal{M}} \mathbb{E}_{\bm{\mathrm{x}}_1,\bm{\mathrm{r}}_1} \left[\left| g_{k,\bm{M}}(\bm{\mathrm{x}}_1,\bm{\mathrm{r}}_1)\right| \cdot \mathds{1}(G(\bm{\mathrm{x}}_1,\bm{\mathrm{r}}_1)> \sqrt{n})  \right]  + \sup_{M\in \mathcal{M}} \frac{1}{n}\sum_{i=1}^{n} \left| g_{k,\bm{M}}(\bm{\mathrm{x}}_i,\bm{\mathrm{r}}_i)\right| \cdot \mathds{1}(G(\bm{\mathrm{x}}_i,\bm{\mathrm{r}}_i)> \sqrt{n})  \\
        &\leq \mathbb{E}_{\bm{\mathrm{x}}_1,\bm{\mathrm{r}}_1} \left[\left| G(\bm{\mathrm{x}}_1,\bm{\mathrm{r}}_1)\right| \cdot \mathds{1}(G(\bm{\mathrm{x}}_1,\bm{\mathrm{r}}_1)> \sqrt{n})  \right] + \frac{1}{n}\sum_{i=1}^{n} \left| G(\bm{\mathrm{x}}_i,\bm{\mathrm{r}}_i)\right| \cdot \mathds{1}(G(\bm{\mathrm{x}}_i,\bm{\mathrm{r}}_i)> \sqrt{n}) ,
    \end{align*}
    then 
    \begin{align*}
        &\mathbb{E}_{\bm{\mathrm{X}},\bm{\mathrm{R}}}\left[ \sup_{M\in \mathcal{M}} \left| \mathbb{E}_{\bm{\mathrm{x}}_1,\bm{\mathrm{r}}_1} \left[ g_{k,\bm{M}}^{+}(\bm{\mathrm{x}}_1,\bm{\mathrm{r}}_1) \right] - \frac{1}{n}\sum_{i=1}^{n} g_{k,\bm{M}}^{+}(\bm{\mathrm{x}}_i,\bm{\mathrm{r}}_i)  \right|  \right] \\
        &\leq 2\cdot \mathbb{E}_{\bm{\mathrm{x}}_1,\bm{\mathrm{r}}_1} \left[\left| G(\bm{\mathrm{x}}_1,\bm{\mathrm{r}}_1)\right| \cdot \mathds{1}(G(\bm{\mathrm{x}}_1,\bm{\mathrm{r}}_1)> \sqrt{n})  \right] \\
        &\leq 2\cdot \left\{ \mathbb{E}_{\bm{\mathrm{x}}_1,\bm{\mathrm{r}}_1} \left[ \left( G(\bm{\mathrm{x}}_1,\bm{\mathrm{r}}_1) \right)^2 \right] \cdot \mathbb{E}_{\bm{\mathrm{x}}_1,\bm{\mathrm{r}}_1} \left[ \mathds{1}(G(\bm{\mathrm{x}}_1,\bm{\mathrm{r}}_1)> \sqrt{n}) \right] \right\}^{1/2} \\
        &=2 \sigma_G \cdot \left\{ \text{Pr}\left( (G(\bm{\mathrm{x}}_1,\bm{\mathrm{r}}_1))^2 > n \right) \right\}^{1/2} \\
        &\leq 2 \sigma_G \cdot \left\{ \frac{ \mathbb{E}_{\bm{\mathrm{x}}_1,\bm{\mathrm{r}}_1}\left[ (G(\bm{\mathrm{x}}_1,\bm{\mathrm{r}}_1))^2 \right] }{ n } \right\}^{1/2} \\
        &= 2 \sigma_G \cdot \frac{\sigma_G }{\sqrt{n} }
        = \frac{2 \sigma_G^2}{\sqrt{n}}.
    \end{align*}
    According to Markov inequality, we obatin 
    \begin{align*}
        &\text{(II)}
        \leq \frac{ 2 \sigma_G^2/\sqrt{n} }{ \tilde{\epsilon}/2 }
        =\frac{ 4\sigma_G^2 }{ \tilde{\epsilon}\sqrt{n} },
    \end{align*}
    which implies that $\text{(II)}\rightarrow 0$ as $n\rightarrow\infty$. 

    Combining $\text{(I)}$ and $\text{(II)}$ leads to 
    \begin{align*}
        &\lim_{n\rightarrow\infty} \text{Pr}\left( \sup_{M\in \mathcal{M}} \left| L^{\lambda}(\bm{M}) - \widehat{L}_n^{\lambda}(\bm{M}) \right| > \tilde{\epsilon} \right) =0,
    \end{align*}
    which completes the proof. 
\end{proof}

\begin{lemma}
\label{lemma_VCdim}
    For the function class 
    \begin{align*}
        \mathcal{G}&=\bigg\{ g_{\bm{\mu}}:(\bm{x},\bm{r})\mapsto \sum_{j=1}^{p} r_j(x_j - \mu_{j})^2 + \lambda(1-r_j)\mu_{j}^2 \;\bigg|\; \forall \bm{\mu} \in \mathcal{X}, \|\bm{\mu}\|_2\leq B \bigg\},
    \end{align*}
    the VC dimension $\text{VCdim}(\mathcal{G})<\infty$. 
\end{lemma}
\begin{proof}
    Define $3p+1$ base functions as follows: 
    \begin{align*}
        &g^{1j}(\bm{x},\bm{r}) = r_j x_j^2,\; j=1,\dots,p\\
        &g^{2j}(\bm{x},\bm{r}) = r_j x_j,\; j=1,\dots,p\\
        &g^{3j}(\bm{x},\bm{r}) = r_j ,\; j=1,\dots,p\\
        &g^{4}(\bm{x},\bm{r}) = 1 .
    \end{align*}
    Because any $g_{\bm{\mu}}\in\mathcal{G}$ can be rewritten as 
    \begin{align*}
        &g_{\bm{\mu}}(\bm{x},\bm{r})\\
        &= \sum_{j=1}^{p} r_j x_j^2 + \sum_{j=1}^{p} (-2\mu_j)\cdot (r_j x_j) + \sum_{j=1}^{p} (1-\lambda)\mu_j^2\cdot r_j + \sum_{j=1}^{p} \lambda \mu_j^2 \\
        &= \sum_{j=1}^{p} g^{1j}(\bm{x},\bm{r}) + \sum_{j=1}^{p} (-2\mu_j)\cdot g^{2j}(\bm{x},\bm{r}) + \sum_{j=1}^{p} (1-\lambda)\mu_j^2\cdot g^{3j}(\bm{x},\bm{r}) + \sum_{j=1}^{p} \lambda \mu_j^2 \cdot g^{4}(\bm{x},\bm{r}),
    \end{align*}
    then it is a linear combination of $\{ g^{11},\dots,g^{1p}, g^{21},\dots,g^{2p}, g^{31}\dots,g^{3p}, g^{4}\}$, i.e., 
    \begin{align*}
        \mathcal{G}\subset \text{span}\left\{ g^{11},\dots,g^{1p}, g^{21},\dots,g^{2p}, g^{31}\dots,g^{3p}, g^{4}\right\}. 
    \end{align*}
    Since the right hand is a vector space of functions given by the linear combination of $3p+1$ base functions, which has a dimension of $3p+1$, then according to Proposition~4.20 of \cite{wainwright2019high}, we have the subgraph class of it has VC dimension at most $3p+1$. 
    It follows that $\text{VCdim}(\mathcal{G})\leq 3p+1<\infty$. 
\end{proof}

\begin{lemma}
\label{lemma_identifiability}
    Under the settings in the main paper, the minimizer of $L^{\lambda}(\cdot)$ in $\mathcal{M}$ is identifiable, that is, for any $\epsilon>0$, $L^{\lambda}(\bm{M}^{\ast,\lambda}) < \inf \{ L^{\lambda}(\bm{M}) \;|\; \|\text{vec}(\bm{M})-\text{vec}(\bm{M}^{\ast,\lambda})\|_2>\epsilon, \bm{M}\in\mathcal{M} \}$. 
\end{lemma}
\begin{proof}
    Through this proof, we let
    \begin{align*}
        m^{\ast,\lambda}=\mathop{\min}_{\bm{M}\in\mathcal{M}} L^{\lambda}(\bm{M})
        \quad\text{and}\quad
        \mathcal{L}=\{ L^{\lambda}(\bm{M})\;|\; \bm{M}\in\mathcal{M}\}.
    \end{align*}
    We first prove: there exists a sequence $\{\bm{V}_n\}_{n\in \mathbb{N}_+}\subset \mathcal{M}$ such that $\lim_{n\rightarrow\infty} L^{\lambda}(\bm{V}_n) = m^{\ast,\lambda}$. 
    To this end, we note that for any $a>m^{\ast,\lambda}$, there exists an $\bm{M}\in \mathcal{M}$ such that $b=L^{\ast}(\bm{M})<a$. 
    Then take $a_n=m^{\ast,\lambda}+1/n$, which means $\lim_{n\rightarrow\infty} a_n=m^{\ast,\lambda}$, and take $\mathcal{L}_n =\{ b \;|\;\forall b\in \mathcal{L}, b<a_n\}$. We have $\mathcal{L}_n\neq \emptyset$. 
    Denote by $\mathfrak{B}(\mathcal{L})$ the power set of $\mathcal{L}$. According to the axiom of choice, there exists a function $g:\mathfrak{B}(\mathcal{L})\setminus \{\emptyset\}\mapsto \mathcal{L}$ such that for any $Bset\in \mathfrak{B}(\mathcal{L})\setminus \{\emptyset\}$, $g(\mathcal{L})\in Bset$. 
    Thus, let $b_n=g(\mathcal{L}_n)$, then we have $b_n\in \mathcal{L}_n$, which means $m^{\ast,\lambda}\leq b_n<a_n$. It follows that $\lim_{n\rightarrow\infty} b_n = m^{\ast,\lambda}$, which implies the existence of the sequence $\{\bm{V}_n\}_{n\in\mathbb{N}_+}\subset \mathcal{M}$ satisfying $\lim_{n\rightarrow\infty} L^{\lambda}(\bm{V}_n) = m^{\ast,\lambda}$. 

    Next, we turn to prove the unique minimizer $\bm{M}^{\ast,\lambda}$ is identifiable by deriving a contradiction. 
    For any $\epsilon>0$, write $\mathcal{M}_{\epsilon}=\{\bm{M}\in \mathcal{M} \;|\; \text{dist}(\bm{M},\bm{M}^{\ast,\lambda})>\epsilon \}$. 
    If there exists an $\epsilon>0$ satisfying 
    \begin{align*}
        m^{\ast,\lambda} = \inf \left\{ L^{\lambda}(\bm{M}) \mid \|\text{vec}(\bm{M})-\text{vec}(\bm{M}^{\ast,\lambda})\|_2 >\epsilon,\; \bm{M}\in \mathcal{M}  \right\},
    \end{align*}
    then by using the same technique in the above proof, we can obtain: there exists a sequence $\{\bm{V}_n\}_{n\in\mathbb{N}}\subset \mathcal{M}_{\epsilon}$ such that $\lim_{n\rightarrow\infty} L^{\lambda}(\bm{V}_n)=m^{\ast,\lambda}$. 
    Moreover, since the space $\mathcal{M}$ is compact, then there exists a convergent subsequence $\{\bm{V}_{n_t}\}_{t\in\mathbb{N}}$ of $\bm{V}_n$, and we can write the limit to be $\bm{V}^{\dagger}=\lim_{t\rightarrow\infty} \bm{V}_{n_t}$. It follows that $L^{\lambda}(\bm{V}^{\dagger})=m^{\ast,\lambda}$. Then, the uniqueness of the minimizer of $L^{\lambda}(\cdot)$ implies $\bm{V}^{\dagger}=\bm{M}^{\ast,\lambda}$. 
    On the other hand, the convergence of $\bm{V}_{n_t}$ to $\bm{V}^{\dagger}$ means that there exists $t_0\in\mathbb{N}$, such that for any $t\geq t_0$, the $\text{dist}(\bm{V}_{n_t},\bm{V}^{\dagger})\leq\epsilon$ holds. It follows that $\|\text{vec}(\bm{V}_{n_t})-\text{vec}(\bm{M}^{\ast,\lambda})\|_2\leq\epsilon$ holds for sufficiently large $t$, which implies that $\bm{V}_{n_t}\notin \mathcal{M}_{\epsilon}$. This is a contradiction, which completes the proof. 
\end{proof}

\begin{lemma}
\label{lemma_proof_converge_to_truth_about_U}
    Under assumptions and settings in Theorem~\ref{theorem_converge_to_truth}, we have given $\widehat{\bm{\mathrm{M}}}^{\lambda}=\bm{M}^{\ast\ast}$, the $\bm{\mathrm{U}}^{\ast\ast}$ minimizes $\widehat{L}_n^{\lambda}(\bm{U},\bm{M}^{\ast\ast})$.
\end{lemma}
\begin{proof}
We define for any $i=1,\dots,n$ and $l=1,\dots,k$,
\begin{align*}
    h_i^{\lambda}(\bm{\mu}_l^{\ast\ast}) = \sum_{j=1}^{p} \mathrm{r}_{ij} ( \mathrm{x}_{ij} - \mu_{lj}^{\ast\ast} )^2 + (1-\mathrm{r}_{ij})\lambda(\mu_{lj}^{\ast\ast})^2,
\end{align*}
then we first prove for any $i=1,\dots,n$: Given $\mathrm{z}_i=l$, $h_i^{\lambda}(\bm{\mu}_{l'}^{\ast\ast}) - h_i^{\lambda}(\bm{\mu}_{l}^{\ast\ast})\geq 0$ holds with a high probability for any $l'\neq l$. 
Write
\begin{align*}
    h_i^{\lambda}(\bm{\mu}_{l'}^{\ast\ast}) - h_i^{\lambda}(\bm{\mu}_{l}^{\ast\ast}) = \sum_{j=1}^{p} \Delta_{ij}^{ll'},
\end{align*}
where 
\begin{align*}
    \Delta_{ij}^{ll'}= \mathrm{r}_{ij}\{ ( \mathrm{x}_{ij} - \mu_{l'j}^{\ast\ast} )^2 - ( \mathrm{x}_{ij} - \mu_{lj}^{\ast\ast} )^2 \} 
    + \lambda (1-\mathrm{r}_{ij})\cdot \{ (\mu_{l'j}^{\ast\ast})^2 - (\mu_{lj}^{\ast\ast})^2\}.
\end{align*}
For any $j=1,\dots,p$ and $l'\neq l$, we have 
\begin{align*}
    &\mathbb{E}_{\bm{\mathrm{x}}_{i}}\left[ \Delta_{ij}^{ll'} \;\middle|\; \mathrm{r}_{ij}=1,\; \mathrm{z}_i=l \right]
    =2(\mu_{lj}^{\ast\ast} - \mu_{l'j}^{\ast\ast} )\cdot \mathbb{E}_{\bm{\mathrm{x}}_{i}}\left[ \mathrm{x}_{ij} \;\middle|\; \mathrm{r}_{ij}=1,\; \mathrm{z}_i=l \right] + (\mu_{l'j}^{\ast\ast})^2 - (\mu_{lj}^{\ast\ast})^2,
\end{align*}
and 
\begin{align*}
    &\text{Var}\left( \Delta_{ij}^{ll'} \;\middle|\; \mathrm{r}_{ij}=1,\; \mathrm{z}_i=l \right)
    =4(\mu_{lj}^{\ast\ast} - \mu_{l'j}^{\ast\ast} )^2\cdot \text{Var}\left( \mathrm{x}_{ij} \;\middle|\; \mathrm{r}_{ij}=1,\; \mathrm{z}_i=l  \right).
\end{align*}
According to Assumption~\ref{assumption_def_MNAR0} (MNAR mechanism), we have 
\begin{align*}
    \text{Pr}(\mathrm{r}_{ij}=0|\mathrm{z}_i=l)
    &=\int_x \text{Pr}(\mathrm{r}_{ij}=0|\mathrm{x}_{ij}=x,\mathrm{z}_i=l)\cdot f(x|\mathrm{z}_i=l)\;dx\\
    &=\int_x \exp(-\lambda^{\ast} x^2)\cdot \frac{1}{\sqrt{2\pi}\sigma} \cdot \exp\left( -\frac{ (x-\mu_{lj}^{\ast\ast} ) }{ 2\sigma^2 } \right) \; dx \\
    &=\frac{1}{\sqrt{2\sigma^2\lambda^{\ast}+1}}\cdot \exp\left( -\frac{ \lambda^{\ast}(\mu_{lj}^{\ast\ast})^2 }{ 2\sigma^2\lambda^{\ast}+1 } \right)\\
    &:=\alpha_{lj}.
\end{align*}
It follows that given $\mathrm{z}_i=l$ and $\mathrm{r}_{ij}=1$, the conditional density function of $\mathrm{x}_{ij}$ can be written as 
\begin{align*}
    f(x|\mathrm{r}_{ij}=1,\; \mathrm{z}_i=l )
    &=\frac{ \text{Pr}(\mathrm{r}_{ij}=1|\mathrm{x}_{ij}=x,\mathrm{z}_i=l)\cdot f(x|\mathrm{z}_i=l) }{ \text{Pr}(\mathrm{r}_{ij}=1 | \mathrm{z}_i=l) } \\
    &=\frac{ (1-\exp(-\lambda^{\ast}x^2))\cdot f(x|\mathrm{z}_i=l) }{ 1-\alpha_{lj} }\\
    &= \frac{1}{1-\alpha_{lj}}\cdot (1-\exp(-\lambda^{\ast}x^2))\cdot\frac{1}{\sqrt{2\pi}\sigma} \cdot \exp\left( -\frac{ (x-\mu_{lj}^{\ast\ast} ) }{ 2\sigma^2 } \right)
\end{align*}
Thus, we can obtain
\begin{align*}
    \mathbb{E}_{\bm{\mathrm{x}}_{i}}\left[ \mathrm{x}_{ij} \;\middle|\; \mathrm{r}_{ij}=1,\; \mathrm{z}_i=l \right] 
    &=\int x\cdot f(x|\mathrm{r}_{ij}=1,\; \mathrm{z}_i=l ) \;dx\\ 
    &= \frac{1}{1-\alpha_{lj}} \cdot \int x\cdot (1-\exp(-\lambda^{\ast} x^2) )\cdot \frac{1}{\sqrt{2\pi}\sigma} \cdot \exp\left( -\frac{ (x-\mu_{lj}^{\ast\ast} ) }{ 2\sigma^2 } \right) \;dx \\
    &=\frac{1}{1-\alpha_{lj}}\cdot \left( \mu_{lj}^{\ast\ast} - \alpha_{lj}\cdot \frac{\mu_{lj}^{\ast\ast} }{ 2\sigma^2\lambda^{\ast}+1 } \right)
\end{align*}
and 
\begin{align*}
    \mathbb{E}_{\bm{\mathrm{x}}_{i}}\left[ \mathrm{x}_{ij}^2 \;\middle|\; \mathrm{r}_{ij}=1,\; \mathrm{z}_i=l \right] 
    &=\int x^2\cdot f(x|\mathrm{r}_{ij}=1,\; \mathrm{z}_i=l ) \;dx\\
    &=\frac{1}{1-\alpha_{lj}} \cdot \int x^2\cdot (1-\exp(-\lambda^{\ast} x^2) )\cdot \frac{1}{\sqrt{2\pi}\sigma} \cdot \exp\left( -\frac{ (x-\mu_{lj}^{\ast\ast} ) }{ 2\sigma^2 } \right) \;dx \\
    &=\frac{ (\mu_{lj}^{\ast\ast})^2 + \sigma^2 }{ 1-\alpha_{lj} } - \frac{ \alpha_{tj} }{ 1- \alpha_{tj} }\cdot\left\{ \left(\frac{ \mu_{lj}^{\ast\ast} }{2\sigma^2\lambda^{\ast}+1 }\right)^2 + \frac{\sigma^2 }{ 2\sigma^2\lambda^{\ast}+1 } \right\}
\end{align*}
and 
\begin{align*}
    &\text{Var}\left( \mathrm{x}_{ij} \;\middle|\; \mathrm{r}_{ij}=1,\; \mathrm{z}_i=l  \right) \\
    &= \mathbb{E}_{\bm{\mathrm{x}}_{i}}\left[ \mathrm{x}_{ij}^2 \;\middle|\; \mathrm{r}_{ij}=1,\; \mathrm{z}_i=l \right] - \left( \mathbb{E}_{\bm{\mathrm{x}}_{i}}\left[ \mathrm{x}_{ij} \;\middle|\; \mathrm{r}_{ij}=1,\; \mathrm{z}_i=l \right] \right)^2 \\
    &= \frac{ (\mu_{lj}^{\ast\ast})^2 + \sigma^2 }{ 1-\alpha_{lj} } - \frac{ \alpha_{tj} }{ 1- \alpha_{tj} }\cdot\left\{ \left(\frac{ \mu_{lj}^{\ast\ast} }{2\sigma^2\lambda^{\ast}+1 }\right)^2 + \frac{\sigma^2 }{ 2\sigma^2\lambda^{\ast}+1 } \right\}   - \frac{1}{(1-\alpha_{lj})^2}\cdot \left( \mu_{lj}^{\ast\ast} - \alpha_{lj}\cdot \frac{\mu_{lj}^{\ast\ast} }{ 2\sigma^2\lambda^{\ast}+1 } \right)^2\\
    &=\frac{\sigma^2}{1-\alpha_{lj}}\cdot\left\{ 1 - \sigma^2\cdot \frac{\alpha_{lj}}{1-\alpha_{lj}}\cdot \left(\frac{ 2\lambda^{\ast}\mu_{lj}^{\ast\ast}}{  2\sigma^2\lambda^{\ast}+1 }\right)^2  \right\}
\end{align*}
Moreover, the $K$-th order moment of $\mathrm{x}_{ij}$ given $\mathrm{z}_i=l$ and $\mathrm{r}_{ij}=1$ is bounded as follows:
\begin{align*}
    &\bigg| \mathbb{E}_{\bm{\mathrm{x}}_i}\left[  \left| \mathrm{x}_{ij} - \mathbb{E}_{\bm{\mathrm{x}}_{i}}\left[ \mathrm{x}_{ij} \;\middle|\; \mathrm{r}_{ij}=1,\; \mathrm{z}_i=l \right] \right|^K \;\middle|\; \mathrm{r}_{ij}=1,\; \mathrm{z}_i=l   \right] \bigg| \\
    &= \bigg| \int \left| x - \mathbb{E}_{\bm{\mathrm{x}}_{i}}\left[ \mathrm{x}_{ij} \;\middle|\; \mathrm{r}_{ij}=1,\; \mathrm{z}_i=l \right] \right|^K \cdot f(x|\mathrm{r}_{ij}=1,\; \mathrm{z}_i=l ) \;dx \bigg| \\
    &\leq \bigg| \int \left| x - \mathbb{E}_{\bm{\mathrm{x}}_{i}}\left[ \mathrm{x}_{ij} \;\middle|\; \mathrm{r}_{ij}=1,\; \mathrm{z}_i=l \right] \right|^K \cdot \frac{1}{1-\alpha_{lj}}\cdot f(x|\mathrm{z}_i=l ) \;dx \bigg| \\
    &= \frac{1}{1-\alpha_{lj}}\cdot \bigg| \mathbb{E}_{\bm{\mathrm{x}}_i}\left[  \left| \mathrm{x}_{ij} - \mathbb{E}_{\bm{\mathrm{x}}_{i}}\left[ \mathrm{x}_{ij} \;\middle|\; \mathrm{r}_{ij}=1,\; \mathrm{z}_i=l \right] \right|^K \;\middle|\; \mathrm{z}_i=l   \right] \bigg|\\
    &\leq \frac{1}{1-\alpha_{lj}}\cdot \frac{1}{2}K!\cdot \sigma^2 \cdot \sigma^{K-2}\cdot \tilde{C}_3,
\end{align*}
where the first inequality is because $ f(x|\mathrm{r}_{ij}=1,\; \mathrm{z}_i=l ) \leq 1/(1-\alpha_{lj})\cdot  f(x|\mathrm{z}_i=l )$ and $f(x|\mathrm{z}_i=l )$ is the density function of $\mathcal{N}(\mu_{lj}^{\ast\ast},\sigma^2)$, and the last inequality is due to the bound of the $K$-th moment of a Gaussian variable and $\tilde{C}_3>0$ is a constant. 

Thereby, we can obtain an upper bound $K$-th moment of $\Delta_{ij}^{ll'}$ as follows: 
\begin{align*}
    &\bigg| \mathbb{E}_{\bm{\mathrm{x}}_i}\left[  \left| \Delta_{ij}^{ll'} - \mathbb{E}_{\bm{\mathrm{x}}_{i}}\left[ \Delta_{ij}^{ll'} \;\middle|\; \mathrm{r}_{ij}=1,\; \mathrm{z}_i=l \right] \right|^K \;\middle|\; \mathrm{r}_{ij}=1,\; \mathrm{z}_i=l   \right] \bigg| \\
    &=2^K\cdot \left| \mu_{lj}^{\ast\ast} - \mu_{l'j}^{\ast\ast} \right|^K \cdot \bigg| \mathbb{E}_{\bm{\mathrm{x}}_i}\left[  \left| \mathrm{x}_{ij} - \mathbb{E}_{\bm{\mathrm{x}}_{i}}\left[ \mathrm{x}_{ij} \;\middle|\; \mathrm{r}_{ij}=1,\; \mathrm{z}_i=l \right] \right|^K \;\middle|\; \mathrm{r}_{ij}=1,\; \mathrm{z}_i=l   \right] \bigg|\\
    &\leq 2^K\cdot \left| \mu_{lj}^{\ast\ast} - \mu_{l'j}^{\ast\ast} \right|^K \cdot \frac{1}{1-\alpha_{lj}}\cdot \frac{1}{2}K!\cdot \sigma^2 \cdot \sigma^{K-2}\cdot \tilde{C}_3.
\end{align*}
By letting
\begin{align*}
    b_{lj}=2C_{lj}\left| \mu_{lj}^{\ast\ast} - \mu_{l'j}^{\ast\ast}  \right|\cdot \sigma,
\end{align*}
where $C_{lj}>0$ is a constant satisfying 
\begin{align*}
    C_{lj}\geq 4\cdot \left| \mu_{lj}^{\ast\ast} - \mu_{l'j}^{\ast\ast} \right|^2\cdot \left\{ 1 - \sigma^2\cdot \frac{\alpha_{lj}}{1-\alpha_{lj}}\cdot \left(\frac{ 2\lambda^{\ast}\mu_{lj}^{\ast\ast}}{  2\sigma^2\lambda^{\ast}+1 }\right)^2  \right\}^{-1} \cdot \tilde{C}_3,
\end{align*}
we can obtain 
\begin{align*}
    &\bigg| \mathbb{E}_{\bm{\mathrm{x}}_i}\left[  \left| \Delta_{ij}^{ll'} - \mathbb{E}_{\bm{\mathrm{x}}_{i}}\left[ \Delta_{ij}^{ll'} \;\middle|\; \mathrm{r}_{ij}=1,\; \mathrm{z}_i=l \right] \right|^K \;\middle|\; \mathrm{r}_{ij}=1,\; \mathrm{z}_i=l   \right] \bigg| \\
    &\leq \frac{1}{2}\cdot K!\cdot \text{Var}\left( \mathrm{x}_{ij} \;\middle|\; \mathrm{r}_{ij}=1,\; \mathrm{z}_i=l  \right) \cdot (b_{lj})^{K-2},
\end{align*}
which means the Bernstein's condition \cite{wainwright2019high} is satisfied. 
Consequently, we can apply the Bernstein inequality (Proposition~2.10 of \cite{wainwright2019high}) to $\Delta_{ij}^{ll'}$ as follows: 
\begin{align*}
    &\text{Pr}\left( \left| \Delta_{ij}^{ll'} - \mathbb{E}_{\bm{\mathrm{x}}_{i}}\left[ \Delta_{ij}^{ll'} \;\middle|\; \mathrm{r}_{ij}=1,\; \mathrm{z}_i=l \right] \right| \leq \epsilon \;\middle|\; \mathrm{r}_{ij}=1,\; \mathrm{z}_i=l \right)\\
    &\geq 1-2\exp\left( -\frac{\epsilon^2}{ 2\text{Var}\left( \mathrm{x}_{ij} \;\middle|\; \mathrm{r}_{ij}=1,\; \mathrm{z}_i=l  \right)  + 2b_{lj}\epsilon } \right),\; \forall \epsilon>0,
\end{align*}
which implies 
\begin{align*}
    &\text{Pr}\left(  \Delta_{ij}^{ll'}  \geq \mathbb{E}_{\bm{\mathrm{x}}_{i}}\left[ \Delta_{ij}^{ll'} \;\middle|\; \mathrm{r}_{ij}=1,\; \mathrm{z}_i=l \right] - \epsilon \;\middle|\; \mathrm{r}_{ij}=1,\; \mathrm{z}_i=l \right)\\
    &\geq 1-2\exp\left( -\frac{\epsilon^2}{ 2\text{Var}\left( \mathrm{x}_{ij} \;\middle|\; \mathrm{r}_{ij}=1,\; \mathrm{z}_i=l  \right)  + 2b_{lj}\epsilon } \right),\; \forall \epsilon>0.
\end{align*}

Furthermore, for a fixed binary vector $\bm{r}\in\{0,1\}^p$, since given $\mathrm{z}_i=l$ and $\bm{\mathrm{r}}_i=\bm{r}$, 
\begin{align*}
    \sum_{j=1}^{p}\Delta_{ij}^{ll'}
    &=\sum_{j:r_j=1} \Delta_{ij}^{ll'} + \sum_{j:r_j=0} \Delta_{ij}^{ll'}\\
    &=\sum_{j:r_j=1} \Delta_{ij}^{ll'} + \lambda\cdot \sum_{j:r_j=0} \{ (\mu_{l'j}^{\ast\ast})^2 - (\mu_{lj}^{\ast\ast})^2\},
\end{align*}
where the second term is a constant, then using the union bound of $\Delta_{ij}^{ll'}|(\mathrm{r}_{ij}=r_j,\;\mathrm{z}_i=l)$ with $j$ satisfying $r_j=1$, we have for any $\epsilon>0$, 
\begin{align*}
    &\text{Pr}\left(  \sum_{j=1}^{p}\Delta_{ij}^{ll'}  \geq \sum_{j:r_j=1} \mathbb{E}_{\bm{\mathrm{x}}_{i}}\left[ \Delta_{ij}^{ll'} \;\middle|\; \mathrm{r}_{ij}=1,\; \mathrm{z}_i=l \right] - \epsilon \|\bm{r}\|_1  +  \lambda\cdot \sum_{j:r_j=0} \{ (\mu_{l'j}^{\ast\ast})^2 - (\mu_{lj}^{\ast\ast})^2\}
    \;\;\middle|\; \bm{\mathrm{r}}_{i}=\bm{r},\; \mathrm{z}_i=l \right)\\
    &\geq 1-2\sum_{j:r_j=1} \exp\left( -\frac{\epsilon^2}{ 2\text{Var}\left( \mathrm{x}_{ij} \;\middle|\; \mathrm{r}_{ij}=1,\; \mathrm{z}_i=l  \right)  + 2b_{lj}\epsilon } \right)\\
    &\geq 1-2p \exp\left( -\frac{\epsilon^2}{ 2\max\limits_{j=1,\dots,p}\text{Var}\left( \mathrm{x}_{ij} \;\middle|\; \mathrm{r}_{ij}=1,\; \mathrm{z}_i=l  \right)  + 2\max\limits_{j=1,\dots,p} b_{lj}\epsilon } \right)
\end{align*}
Notice that 
\begin{align*}
    &\mathbb{E}_{\bm{\mathrm{x}}_{i}}\left[ \Delta_{ij}^{ll'} \;\middle|\; \mathrm{r}_{ij}=1,\; \mathrm{z}_i=l \right]\\
    &=2\cdot \left( \mu_{lj}^{\ast\ast} - \mu_{l'j}^{\ast\ast} \right) \cdot \frac{1}{1-\alpha_{lj}}\cdot \left( \mu_{lj}^{\ast\ast} - \alpha_{lj}\cdot \frac{\mu_{lj}^{\ast\ast} }{ 2\sigma^2\lambda^{\ast}+1 } \right) + (\mu_{l'j}^{\ast\ast})^2 - (\mu_{lj}^{\ast\ast})^2\\
    &= 2\cdot \left( \mu_{lj}^{\ast\ast} - \mu_{l'j}^{\ast\ast} \right) \cdot \left\{ \mu_{lj}^{\ast\ast}\cdot \left( \frac{1}{1-\alpha_{lj}} - \frac{\alpha_{lj}}{1-\alpha_{lj}}\cdot \frac{1}{ 2\sigma^2\lambda^{\ast}+1 } \right)  - \frac{ \mu_{lj}^{\ast\ast} + \mu_{l'j}^{\ast\ast} }{2}  \right\},
\end{align*}
then $\mathbb{E}_{\bm{\mathrm{x}}_{i}}\left[ \Delta_{ij}^{ll'} \;\middle|\; \mathrm{r}_{ij}=1,\; \mathrm{z}_i=l \right]$ tends to $(\mu_{lj}^{\ast\ast} - \mu_{l'j}^{\ast\ast})^2$ as $\sigma^2\rightarrow 0$. 
Under the condition of $\lambda$ (i.e., $\lambda=1-1/(2\sigma^2\lambda^{\ast}+1)$), we have $\lambda\cdot \sum_{j:r_j=0} \{ (\mu_{l'j}^{\ast\ast})^2 - (\mu_{lj}^{\ast\ast})^2\}$ tends to zero as $\sigma^2\rightarrow 0$. 
It means that take $\epsilon=\epsilon_{ll'}^{\bm{r}}$, where
\begin{align*}
    \epsilon_{ll'}^{\bm{r}}=\frac{\|\bm{r}\|_1 }{2} \cdot \mathbb{E}_{\bm{\mathrm{x}}_{i}}\left[ \Delta_{ij}^{ll'} \;\middle|\; \mathrm{r}_{ij}=1,\; \mathrm{z}_i=l \right] + \frac{\|\bm{r}\|_1 \lambda }{2}\cdot \sum_{j:r_j=0} \{ (\mu_{l'j}^{\ast\ast})^2 - (\mu_{lj}^{\ast\ast})^2\} ,
\end{align*}
and $\epsilon_{ll'}^{\bm{r}}>0$ holds for a sufficiently small $\sigma^2$, which follows that
\begin{align*}
    &\text{Pr}\left(  \sum_{j=1}^{p}\Delta_{ij}^{ll'}  > 0
    \;\middle|\; \bm{\mathrm{r}}_{i}=\bm{r},\; \mathrm{z}_i=l \right)
    \geq 1-2p \exp\left( -\frac{(\epsilon_{ll'}^{\bm{r}})^2}{ 2\max\limits_{j=1,\dots,p}\text{Var}\left( \mathrm{x}_{ij} \;\middle|\; \mathrm{r}_{ij}=1,\; \mathrm{z}_i=l  \right)  + 2\max\limits_{j=1,\dots,p} b_{lj}\epsilon_{ll'}^{\bm{r}} } \right).
\end{align*}
In addition, because
\begin{align*}
    \text{Pr}\left( \sum_{j=1}^{p}\Delta_{ij}^{ll'}  \leq 0 \;\middle|\;\mathrm{z}_i=l \right) 
    &=\sum_{\bm{r}\in\{0,1\}^p } \text{Pr}\left(\bm{\mathrm{r}}_{i}=\bm{r}\;\middle|\; \mathrm{z}_i=l \right)\cdot  \text{Pr}\left(  \sum_{j=1}^{p}\Delta_{ij}^{ll'}  \leq 0
    \;\middle|\; \bm{\mathrm{r}}_{i}=\bm{r},\; \mathrm{z}_i=l \right)\\
    &\leq 2^p\cdot \max_{\bm{r}\in\{0,1\}^p } \text{Pr}\left(  \sum_{j=1}^{p}\Delta_{ij}^{ll'}  \leq 0
    \;\middle|\; \bm{\mathrm{r}}_{i}=\bm{r},\; \mathrm{z}_i=l \right),
\end{align*}
then we obtain
\begin{align*}
    &\text{Pr}\bigg( h_i^{\lambda}(\bm{\mu}_{l'}^{\ast\ast}) - h_i^{\lambda}(\bm{\mu}_{l}^{\ast\ast})\geq 0 \;\bigg|\;\mathrm{z}_i=l \bigg) \\
    &\geq\text{Pr}\left( \sum_{j=1}^{p}\Delta_{ij}^{ll'}  > 0 \;\middle|\;\mathrm{z}_i=l \right) \\
    &\geq 1- 2^{p+1}p\cdot  \max_{\bm{r}\in\{0,1\}^p } \exp\left( -\frac{(\epsilon_{ll'}^{\bm{r}})^2}{ 2\max\limits_{j=1,\dots,p}\text{Var}\left( \mathrm{x}_{ij} \;\middle|\; \mathrm{r}_{ij}=1,\; \mathrm{z}_i=l  \right)  + 2\max\limits_{j=1,\dots,p} b_{lj}\epsilon_{ll'}^{\bm{r}} } \right).
\end{align*}

Next, denote by $\hat{\mathrm{z}}_i^{\lambda}$ the estimated cluster label of $i$-th sample point, i.e., 
\begin{align*}
    \hat{\mathrm{z}}_i^{\lambda}
    &=\mathop{\arg\min}_{t=1,\dots,k}h_i^{\lambda}(\bm{\mu}_{t}^{\ast\ast})\\
    &=\mathop{\arg\min}_{t=1,\dots,k} \sum_{j=1}^{p} \mathrm{r}_{ij} ( \mathrm{x}_{ij} - \mu_{tj}^{\ast\ast} )^2 + (1-\mathrm{r}_{ij})\lambda(\mu_{tj}^{\ast\ast})^2 .
\end{align*}
Because for any $i=1,\dots,n$,
\begin{align*}
    \text{Pr}\left( \hat{\mathrm{z}}_i^{\lambda}\neq\mathrm{z}_i\right)
    &=\sum_{l=1}^{k} \text{Pr}\left( \mathrm{z}_i= l \right) \cdot \text{Pr}\left( \hat{\mathrm{z}}_i^{\lambda}\neq l \;\middle|\;\mathrm{z}_i=l \right)\\
    &=\sum_{l=1}^{k} \frac{1}{k} \cdot \sum_{l'\neq l} \text{Pr}\left( \hat{\mathrm{z}}_i^{\lambda}= l' \;\middle|\;\mathrm{z}_i=l \right)\\
    &=\sum_{l=1}^{k} \frac{1}{k} \cdot \sum_{l'\neq l} \text{Pr}\left( h_i^{\lambda}(\bm{\mu}_{l'}^{\ast\ast}) - h_i^{\lambda}(\bm{\mu}_{l}^{\ast\ast}) < 0 \;\middle|\;\mathrm{z}_i=l \right) \\
    &\leq (k-1)\cdot \max_{\substack{l=1,\dots,k\\ l'\neq l}}  \text{Pr}\left( h_i^{\lambda}(\bm{\mu}_{l'}^{\ast\ast}) - h_i^{\lambda}(\bm{\mu}_{l}^{\ast\ast}) < 0 \;\middle|\;\mathrm{z}_i=l \right)
\end{align*}
then we have 
\begin{align*}
    &\text{Pr}\left( \forall i=1,\dots,n,\; \hat{\mathrm{z}}_i^{\lambda}=\mathrm{z}_i \right) \\
    &\geq 1- 2^{p+1}p(k-1)n\cdot  \max_{\substack{l=1,\dots,k\\ l'\neq l}} \max_{\bm{r}\in\{0,1\}^p } \exp\left( -\frac{(\epsilon_{ll'}^{\bm{r}})^2}{ 2\max\limits_{j=1,\dots,p}\text{Var}\left( \mathrm{x}_{ij} \;\middle|\; \mathrm{r}_{ij}=1,\; \mathrm{z}_i=l  \right)  + 2\max\limits_{j=1,\dots,p} b_{lj}\epsilon_{ll'}^{\bm{r}} } \right).
\end{align*}

Finally, since for any $i,j$, both $\text{Var}\left( \mathrm{x}_{ij} \;\middle|\; \mathrm{r}_{ij}=1,\; \mathrm{z}_i=l  \right)$ and $b_{lj}$ tend to zero as $\sigma^2\rightarrow 0$, more specifically, $\text{Var}\left( \mathrm{x}_{ij} \;\middle|\; \mathrm{r}_{ij}=1,\; \mathrm{z}_i=l  \right)=O(\sigma^2)$ and $b_{lj}=O(\sigma)$, then under the condition $\sigma^2 \log(n)\rightarrow 0$, we obtain that $\text{Pr}\left( \forall i=1,\dots,n,\; \hat{\mathrm{z}}_i^{\lambda}=\mathrm{z}_i \right)\rightarrow 1$.

On the other hand, for fully observed data, recall 
\begin{align*}
    \mathrm{z}_i^{\ast\ast}=\mathop{\arg\min}_{l=1,\dots,k} \|\bm{\mathrm{x}}_i -\bm{\mu}_l^{\ast\ast} \|_2^2.
\end{align*}
Because 
\begin{align*}
    &\|\bm{\mathrm{x}}_i -\bm{\mu}_{l'}^{\ast\ast} \|_2^2 - \|\bm{\mathrm{x}}_i -\bm{\mu}_l^{\ast\ast} \|_2^2
    = \sum_{j=1}^{p} 2(\mu_{lj}^{\ast\ast} - \mu_{l'j}^{\ast\ast})\mathrm{x}_{ij} + (\mu_{l'j}^{\ast\ast})^2 - (\mu_{lj}^{\ast\ast})^2, 
\end{align*}
and given $\mathrm{z}_i=l$, $\mathrm{x}_{ij}\sim \mathcal{N}(\mu_{lj}^{\ast\ast},\sigma^2)$, then by using the same technique, we can obtain $\text{Pr}\left( \forall i=1,\dots,n,\; \mathrm{z}_i^{\ast\ast}=\mathrm{z}_i \right)\rightarrow 1$ under the condition $\sigma^2 \log(n)\rightarrow 0$. 
Combining the consistency results of $\hat{\mathrm{z}}_i^{\lambda}$ and $\mathrm{z}_i^{\ast\ast}$ leads to $\text{Pr}\left( \forall i=1,\dots,n,\; \hat{\mathrm{z}}_i^{\lambda}=\mathrm{z}_i^{\ast\ast} \right)\rightarrow 1$ under the condition $\sigma^2 \log(n)\rightarrow 0$. 
It implies that $\bm{\mathrm{U}}^{\ast\ast}$ minimizes $\widehat{L}_n^{\lambda}(\bm{U},\bm{\mathrm{M}}^{\ast\ast})$ as $n\rightarrow\infty$. 

\end{proof}

\begin{lemma}
\label{lemma_proof_converge_to_truth_about_M}
    Under assumptions and settings in Theorem~\ref{theorem_converge_to_truth}, we have given $\widehat{\bm{\mathrm{U}}}^{\lambda}=\bm{\mathrm{U}}^{\ast\ast}$, the $\widehat{\bm{\mathrm{M}}}^{\lambda}$ converges to $\bm{M}^{\ast\ast}$. 
\end{lemma}
\begin{proof}
Write the empirical loss function of $k$-means on fully observed data to be 
\begin{align*}
    \widehat{L}_n^{\text{full}}(\bm{U},\bm{M}) = \frac{1}{n}\sum_{i=1}^{n} \sum_{l=1}^{k} u_{il} \sum_{j=1}^{p} ( \mathrm{x}_{ij} - \mu_{lj} )^2,
\end{align*}
and denote by $(\widehat{\bm{\mathrm{U}}}^{\text{full}},\widehat{\bm{\mathrm{M}}}^{\text{full}})$ the unique minimizer of $\widehat{L}_n^{\text{full}}(\bm{U},\bm{M})$. According to the consistency of $k$-means on fully observed data \cite{pollard1981strong}, we have $\widehat{\bm{\mathrm{M}}}^{\text{full}}$ converges to $\bm{M}^{\ast\ast}$ in probability as $n\rightarrow\infty$. 

Since the $\bm{M}^{\ast\ast}$ minimizes $\mathbb{E}_{\bm{\mathrm{X}},\bm{\mathrm{R}} }\left[ \widehat{L}_n^{\text{full}}(\bm{\mathrm{U}}^{\ast\ast},\bm{M}) \right]$, then it suffices to prove that the gradient of $\mathbb{E}_{\bm{\mathrm{X}},\bm{\mathrm{R}}}\left[ \widehat{L}_n^{\text{full}}(\bm{\mathrm{U}}^{\ast\ast},\bm{M}) - \widehat{L}_n^{\lambda}(\bm{\mathrm{U}}^{\ast\ast},\bm{M})  \right]$ about $\bm{M}$ is zero at $\bm{M}=\bm{M}^{\ast\ast}$. 
First, we have
\begin{align*}
    &\mathbb{E}_{\bm{\mathrm{X}},\bm{\mathrm{R}}}\left[ \widehat{L}_n^{\text{full}}(\bm{\mathrm{U}}^{\ast\ast},\bm{M}) - \widehat{L}_n^{\lambda}(\bm{\mathrm{U}}^{\ast\ast},\bm{M})  \right]\\
    &=\frac{1}{n}\sum_{i=1}^{n} \mathbb{E}_{\bm{\mathrm{X}},\bm{\mathrm{R}} } \left[ \sum_{l=1}^{k}\mathrm{u}_{il}^{\ast\ast} \sum_{j=1}^{p} (1-\mathrm{r}_{ij}) \left\{ (\mathrm{x}_{ij} - \mu_{lj} )^2 - \lambda \mu_{lj}^2 \right\} \right] \\
    &=\frac{1}{n}\sum_{i=1}^{n} \mathbb{E}_{\bm{\mathrm{X}} } \left[ \sum_{j=1}^{p} \mathrm{E}_{\bm{\mathrm{R}}}\left[ 1-\mathrm{r}_{ij}  \;\middle|\; \bm{\mathrm{X}} \right] \cdot \sum_{l=1}^{k}\mathrm{u}_{il}^{\ast\ast} \left\{ (\mathrm{x}_{ij} - \mu_{lj} )^2 - \lambda \mu_{lj}^2 \right\} \right] \\
    &= \frac{1}{n}\sum_{i=1}^{n} \mathbb{E}_{\bm{\mathrm{X}} } \left[ \sum_{j=1}^{p} \text{Pr}\left( \mathrm{r}_{ij} =0  \;\middle|\; \bm{\mathrm{X}} \right) \cdot \sum_{l=1}^{k}\mathrm{u}_{il}^{\ast\ast} \left\{ (\mathrm{x}_{ij} - \mu_{lj} )^2 - \lambda \mu_{lj}^2 \right\} \right] \\
    &= \frac{1}{n}\sum_{i=1}^{n} \mathbb{E}_{\bm{\mathrm{x}}_i } \left[ \sum_{j=1}^{p} \exp(-\lambda^{\ast}\mathrm{x}_{ij}^2 ) \cdot \sum_{l=1}^{k}\mathrm{u}_{il}^{\ast\ast} \left\{ (\mathrm{x}_{ij} - \mu_{lj} )^2 - \lambda \mu_{lj}^2 \right\} \right]\\
    &= \frac{1}{n}\sum_{i=1}^{n}\sum_{l=1}^{k} \mathbb{E}_{\bm{\mathrm{x}}_i } \left[ \mathrm{u}_{il}^{\ast\ast} \cdot \sum_{j=1}^{p} \exp(-\lambda^{\ast}\mathrm{x}_{ij}^2 ) \cdot \left\{ (\mathrm{x}_{ij} - \mu_{lj} )^2 - \lambda \mu_{lj}^2 \right\} \right]\\
    &= \frac{1}{n}\sum_{i=1}^{n}\sum_{l=1}^{k} \mathbb{E}_{\bm{\mathrm{x}}_i } \left[ \mathds{1}(\mathrm{z}_i=l) \cdot \sum_{j=1}^{p} \exp(-\lambda^{\ast}\mathrm{x}_{ij}^2 ) \cdot \left\{ (\mathrm{x}_{ij} - \mu_{lj} )^2 - \lambda \mu_{lj}^2 \right\} \right]\\
    &\quad + \frac{1}{n}\sum_{i=1}^{n}\sum_{l=1}^{k} \mathbb{E}_{\bm{\mathrm{x}}_i } \left[ \left\{ \mathrm{u}_{il}^{\ast\ast} - \mathds{1}(\mathrm{z}_i=l) \right\} \cdot \sum_{j=1}^{p} \exp(-\lambda^{\ast}\mathrm{x}_{ij}^2 ) \cdot \left\{ (\mathrm{x}_{ij} - \mu_{lj} )^2 - \lambda \mu_{lj}^2 \right\} \right]\\
    &:= \text{(I)} + \text{(II)}. 
\end{align*}
For $\text{(I)}$, we have 
\begin{align*}
    &\text{(I)}\\
    &=\frac{1}{n}\sum_{i=1}^{n}\sum_{l=1}^{k} \mathbb{E}_{\bm{\mathrm{x}}_i } \left[ \mathds{1}(\mathrm{z}_i=l) \cdot \sum_{j=1}^{p} \exp(-\lambda^{\ast}\mathrm{x}_{ij}^2 ) \cdot \left\{ (\mathrm{x}_{ij} - \mu_{lj} )^2 - \lambda \mu_{lj}^2 \right\} \right]\\
    &=\frac{1}{n}\sum_{i=1}^{n}\sum_{l=1}^{k}  \text{Pr}(\mathrm{z}_i=l)\cdot  \mathbb{E}_{\bm{\mathrm{x}}_i }\left[ \sum_{j=1}^{p} \exp(-\lambda^{\ast}\mathrm{x}_{ij}^2 ) \cdot \left\{ (\mathrm{x}_{ij} - \mu_{lj} )^2 - \lambda \mu_{lj}^2 \right\} \;\middle|\; \mathrm{z}_i=l \right]  \\
    &=\frac{1}{n}\sum_{i=1}^{n}\sum_{l=1}^{k}  \frac{1}{k}\cdot \sum_{j=1}^{p} \mathbb{E}_{\bm{\mathrm{x}}_i }\left[ \exp(-\lambda^{\ast}\mathrm{x}_{ij}^2 ) \cdot \left\{ (\mathrm{x}_{ij} - \mu_{lj} )^2 - \lambda \mu_{lj}^2 \right\} \;\middle|\; \mathrm{z}_i=l \right] \\
    &:= \frac{1}{n}\sum_{i=1}^{n}\sum_{l=1}^{k} \frac{1}{k} \cdot \sum_{j=1}^{p} g_i(\mu_{lj}). 
\end{align*}

For any $i$ satisfying $\mathrm{z}_i=l$, according to construction of $\bm{\mathrm{x}}_i$, we have $\bm{\mathrm{x}}_i$ follows the normal distribution $\mathcal{N}(\bm{\mu}_l^{\ast\ast},\sigma^2\bm{I}_p )$ and each component $\mathrm{x}_{ij}$ is independent and follows $\mathcal{N}(\mu_{lj}^{\ast\ast},\sigma^2 )$. 
Consider any fixed $i=1,\dots,n$, $l=1,\dots,k$ and $j=1,\dots,p$, we have 
\begin{align*}
    g_i(\mu_{lj})
    &= \mathbb{E}_{\bm{\mathrm{x}}_i }\left[ \exp(-\lambda^{\ast}\mathrm{x}_{ij}^2 ) \cdot \left\{ (\mathrm{x}_{ij} - \mu_{lj} )^2 - \lambda \mu_{lj}^2 \right\} \;\middle|\; \mathrm{z}_i=l \right] \\
    &=\int \exp(-\lambda^{\ast} x^2 )\cdot \left\{ (x - \mu_{lj} )^2 - \lambda \mu_{lj}^2 \right\}\cdot \frac{1}{\sqrt{2\pi} \sigma} \cdot \exp\left( -\frac{(x - \mu_{lj}^{\ast\ast})^2}{2\sigma^2}  \right)   dx\\
    &= \frac{1}{\sqrt{2\sigma^2\lambda^{\ast}+1 }} \cdot \exp\left( -\frac{ \lambda^{\ast} (\mu_{lj}^{\ast\ast})^2 }{ 2\sigma^2\lambda^{\ast}+1  } \right) \cdot \int \left\{ (x - \mu_{lj} )^2 - \lambda \mu_{lj}^2 \right\} \\ 
     &\quad \cdot \frac{ \sqrt{2\sigma^2\lambda^{\ast}+1 } }{\sqrt{2\pi}\sigma } \exp\left( -\frac{ 2\sigma^2\lambda^{\ast}+1 }{ 2\sigma^2 } \left( x - \frac{\mu_{lj}^{\ast\ast}  }{ 2\sigma^2\lambda^{\ast}+1 } \right)^2 \right)  \;dx \\
     &= \frac{1}{\sqrt{2\sigma^2\lambda^{\ast}+1 }} \cdot \exp\left( -\frac{ \lambda^{\ast} (\mu_{lj}^{\ast\ast})^2 }{ 2\sigma^2\lambda^{\ast}+1  } \right) \cdot \mathbb{E}_{\mathrm{y}_{ij}} \left[ (\mathrm{y}_{ij} - \mu_{lj})^2 - \lambda \mu_{lj}^2 \right],
\end{align*}
where $\mathrm{y}_{ij}$ is a random variable following a normal distribution 
\begin{align*}
    \mathcal{N}\left( \frac{\mu_{lj}^{\ast\ast}  }{ 2\sigma^2\lambda^{\ast}+1 } \;,\; \frac{\sigma^2}{ 2\sigma^2\lambda^{\ast}+1 } \right). 
\end{align*}
then we have 
\begin{align*}
    \frac{\partial g_i }{ \partial \mu_{lj} } 
    &= \frac{1}{\sqrt{2\sigma^2\lambda^{\ast}+1 }} \cdot \exp\left( -\frac{ \lambda^{\ast} (\mu_{lj}^{\ast\ast})^2 }{ 2\sigma^2\lambda^{\ast}+1  } \right) \cdot \mathbb{E}_{\mathrm{y}_{ij}} \left[ -2 \mathrm{y}_{ij} -2 (\lambda -1)\mu_{lj} ) \right] \\
    &= \frac{1}{\sqrt{2\sigma^2\lambda^{\ast}+1 }} \cdot \exp\left( -\frac{ \lambda^{\ast} (\mu_{lj}^{\ast\ast})^2 }{ 2\sigma^2\lambda^{\ast}+1  } \right) \cdot (-2)\cdot \left( \frac{\mu_{lj}^{\ast\ast}  }{ 2\sigma^2\lambda^{\ast}+1 } + (\lambda -1)\mu_{lj}  \right).
\end{align*}
Thus, under the condition $\lambda = 1 - 1/(2\sigma^2\lambda^{\ast}+1 )$, we have the gradient of $g_i $ is zero at $\mu_{lj}= \mu_{lj}^{\ast\ast}$, which follows that the gradient of $\text{(I)}$ about $\bm{M}$ is zero at $\bm{M}=\bm{M}^{\ast\ast}$.

For $\text{(II)}$, we have 
\begin{align*}
    \text{(II)}
    &=\frac{1}{n}\sum_{i=1}^{n} \sum_{l=1}^{k} \text{Pr}(\mathrm{z}_i=l)\cdot \mathbb{E}_{\bm{\mathrm{x}}_i}\left[ (\mathrm{u}_{il}^{\ast\ast} - 1  ) \cdot  \sum_{j=1}^{p} \exp(-\lambda^{\ast}\mathrm{x}_{ij}^2 ) \cdot \left\{ (\mathrm{x}_{ij} - \mu_{lj} )^2 - \lambda \mu_{lj}^2 \right\} \middle|\; \mathrm{z}_i=l \right] \\
    &=\frac{1}{n}\sum_{i=1}^{n} \sum_{l=1}^{k} \frac{1}{k}\cdot \mathbb{E}_{\bm{\mathrm{x}}_i}\left[ (\mathrm{u}_{il}^{\ast\ast} - 1  ) \cdot \sum_{j=1}^{p} \exp(-\lambda^{\ast}\mathrm{x}_{ij}^2 ) \cdot \left\{ (\mathrm{x}_{ij} - \mu_{lj} )^2 - \lambda \mu_{lj}^2 \right\} \;\middle|\; \mathrm{z}_i=l \right] \\
    &\leq \frac{1}{n}\sum_{i=1}^{n} \sum_{l=1}^{k} \frac{1}{k}\cdot \left\{ \mathbb{E}_{\bm{\mathrm{x}}_i}\left[ (\mathrm{u}_{il}^{\ast\ast} - 1  )^2 \;\middle|\; \mathrm{z}_i=l \right] \right\}^{1/2} \cdot \\
    &\quad \left\{ \mathbb{E}_{\bm{\mathrm{x}}_i}\left[ \left( \sum_{j=1}^{p} \exp(-\lambda^{\ast}\mathrm{x}_{ij}^2 ) \cdot \left\{ (\mathrm{x}_{ij} - \mu_{lj} )^2 - \lambda \mu_{lj}^2 \right\} \right)^2  \;\middle|\; \mathrm{z}_i=l \right] \right\}^{1/2}
\end{align*}
Because 
\begin{align*}
    &\mathbb{E}_{\bm{\mathrm{x}}_i}\left[ (\mathrm{u}_{il}^{\ast\ast} - 1  )^2 \;\middle|\; \mathrm{z}_i=l  \right] 
    =\text{Pr}\left(\mathrm{u}_{il}^{\ast\ast}=0 \;\middle|\; \mathrm{z}_i=l \right)
    =\text{Pr}\left(  \mathrm{z}_i^{\ast\ast}\neq l \;\middle|\; \mathrm{z}_i=l  \right) ,
\end{align*}
and 
\begin{align*}
    &\mathbb{E}_{\bm{\mathrm{x}}_i}\left[ \left( \sum_{j=1}^{p} \exp(-\lambda^{\ast}\mathrm{x}_{ij}^2 ) \cdot \left\{ (\mathrm{x}_{ij} - \mu_{lj} )^2 - \lambda \mu_{lj}^2 \right\} \right)^2  \;\middle|\; \mathrm{z}_i=l \right] \\
    &\leq \mathbb{E}_{\bm{\mathrm{x}}_i}\left[ \left( \sum_{j=1}^{p} 1 \cdot \left| (\mathrm{x}_{ij} - \mu_{lj} )^2 - \lambda \mu_{lj}^2 \right| \right)^2  \;\middle|\; \mathrm{z}_i=l \right] \\
    &\leq \mathbb{E}_{\bm{\mathrm{x}}_i}\left[ p\cdot \sum_{j=1}^{p} \left\{ (\mathrm{x}_{ij} - \mu_{lj} )^2 - \lambda \mu_{lj}^2 \right\}^2 \;\middle|\; \mathrm{z}_i=l \right] \\
    &\leq C_4,
\end{align*}
where $C_4<\infty$ is a constant under the Assumption~\ref{assumption_X_compact_sub-Gaussian} (i.e., $\bm{\mathrm{x}}_i$ is sub-Gaussian), then  
\begin{align*}
    \text{(II)}
    &\leq \frac{1}{n}\sum_{i=1}^{n} \sum_{l=1}^{k} \frac{1}{k}\cdot \left\{ \text{Pr}\left(  \mathrm{z}_i^{\ast\ast}\neq l \;\middle|\; \mathrm{z}_i=l  \right)\right\}^{1/2} \cdot \sqrt{C_4} \\
    &\leq \sqrt{C_4}\cdot \frac{1}{n}\sum_{i=1}^{n} \left\{ \sum_{l=1}^{k} \frac{1}{k}\cdot \text{Pr}\left(  \mathrm{z}_i^{\ast\ast}\neq l \;\middle|\; \mathrm{z}_i=l  \right) \right\}^{1/2} \\
    &= \sqrt{C_4}\cdot \frac{1}{n}\sum_{i=1}^{n} \left\{ \text{Pr}\left( \mathrm{z}_i^{\ast\ast}\neq   \mathrm{z}_i \right)\right\}^{1/2}\\
    &\leq \sqrt{C_4}\cdot \left\{ \frac{1}{n}\sum_{i=1}^{n} \text{Pr}\left( \mathrm{z}_i^{\ast\ast}\neq   \mathrm{z}_i \right) \right\}^{1/2}
\end{align*}
Since for any fixed $i=1,\dots,n$,
\begin{align*}
    \text{Pr}\left( \mathrm{z}_i^{\ast\ast}\neq   \mathrm{z}_i \right) 
    &\leq \text{Pr}\left( \text{there exists some } i'\in\{1,\dots,n\} \text{ s.t }  \mathrm{z}_{i'}^{\ast\ast}\neq   \mathrm{z}_{i'}  \right)\\
    &=1-\text{Pr}\left( \forall i'=1,\dots,n,\;  \mathrm{z}_{i'}^{\ast\ast} =   \mathrm{z}_{i'} \right)
\end{align*}
then we obtain
\begin{align*}
    &\text{(II)}
    \leq \sqrt{C_4}\cdot \big\{ 1-\text{Pr}\left( \forall i'=1,\dots,n,\;  \mathrm{z}_{i'}^{\ast\ast} =   \mathrm{z}_{i'} \right) \big\}^{1/2}.
\end{align*}
According to the result about $\mathrm{z}_i^{\ast\ast}$ in the proof of Lemma~\ref{lemma_proof_converge_to_truth_about_U}, it holds that $\text{Pr}\left( \forall i'=1,\dots,n,\;  \mathrm{z}_{i'}^{\ast\ast} =   \mathrm{z}_{i'} \right)\rightarrow 1$ under the condition $\sigma^2\log(n)\rightarrow 0$. 
Thus, we obtain $\text{(II)}\rightarrow 0$, which follows that the gradient of $\text{(II)}$ about $\bm{M}$ tends to zero at $\bm{M}=\bm{M}^{\ast\ast}$.

Therefore, combining $\text{(I)}$ and $\text{(II)}$, we have $\bm{M}^{\ast\ast}$ asymptotically minimizes $\mathbb{E}_{\bm{\mathrm{X}},\bm{\mathrm{R}}}\left[ \widehat{L}_n^{\text{full}}(\bm{\mathrm{U}}^{\ast\ast},\bm{M}) - \widehat{L}_n^{\lambda}(\bm{\mathrm{U}}^{\ast\ast},\bm{M})  \right]$ with $\lambda = 1 - 1/(2\sigma^2\lambda^{\ast}+1 )$ in the high-SNR regime $\sigma^2\log(n)\rightarrow 0$. 
Since $\mathbb{E}_{\bm{\mathrm{X}},\bm{\mathrm{R}}}\left[ \widehat{L}_n^{\text{full}}(\bm{\mathrm{U}}^{\ast\ast},\bm{M})\right]$ is minimized at $\bm{M}=\bm{M}^{\ast\ast}$, then we have $\bm{M}^{\ast\ast}$ is an asymptotic minimizer of $\mathbb{E}_{\bm{\mathrm{X}},\bm{\mathrm{R}}}\left[ \widehat{L}_n^{\lambda}(\bm{\mathrm{U}}^{\ast\ast},\bm{M})  \right]$. 
According to uniform convergence of $\widehat{L}_n^{\lambda}$ (Lemma~\ref{lemma_uniform_convergence}) and identifiability of its minimizer (Lemma~\ref{lemma_identifiability}), we can immediately obtain that $\widehat{\bm{\mathrm{M}}}^{\lambda}$ with $\lambda = 1 - 1/(2\sigma^2\lambda^{\ast}+1 )$ converges to $\bm{M}^{\ast\ast}$ as $n\rightarrow\infty$ in the high-SNR regime $\sigma^2\log(n)\rightarrow 0$.

\end{proof}

\end{document}